\pdfoutput=1

\documentclass[11pt]{article}

\usepackage{EMNLP2023}

\usepackage{times}
\usepackage{latexsym}
\usepackage{booktabs}
\usepackage{multirow}
\usepackage{graphicx}
\usepackage{amsmath}
\usepackage{amssymb}
\usepackage{amsthm}
\newtheorem{theorem}{Theorem}
\newtheorem{corollary}[theorem]{Corollary}
\newtheorem{proposition}{Proposition}
\usepackage{xcolor}
\usepackage{colortbl}
\usepackage{tikz}
\usepackage{tcolorbox}
\tcbuselibrary{breakable}
\usepackage{pgfplots}
\pgfplotsset{compat=1.17}
\usepackage{booktabs}
\usepackage{placeins}
\usepackage[T1]{fontenc}
\usepackage[utf8]{inputenc}
\usepackage{microtype}
\usepackage{inconsolata}
\usepackage[inline]{enumitem}
\setlist{nosep,leftmargin=*}
\newcommand{\method}{\textsc{ActionRating}}

\newcommand{\tgray}[1]{\textcolor{gray}{#1}}
\newcommand{\tgreen}[1]{\textcolor{green!50!black}{#1}}
\newcommand{\gain}[1]{\tgreen{\textbf{#1}}}
\title{Knowing When to Ask:\\ Self-Gated Clarification for Hierarchical Language Agents}

\author{%
  Aijing Gao, 
  Yiming Kang, 
  Mengdie Flora Wang, 
  Jae Oh Woo \\[0.5em]
  Amazon Web Services \\
  \texttt{\{gaijing, ymkang, florawan, jaeohwoo\}@amazon.com}
}

\begin{document}
\maketitle

\begin{figure*}[t!]
\centering
\includegraphics[width=\textwidth]{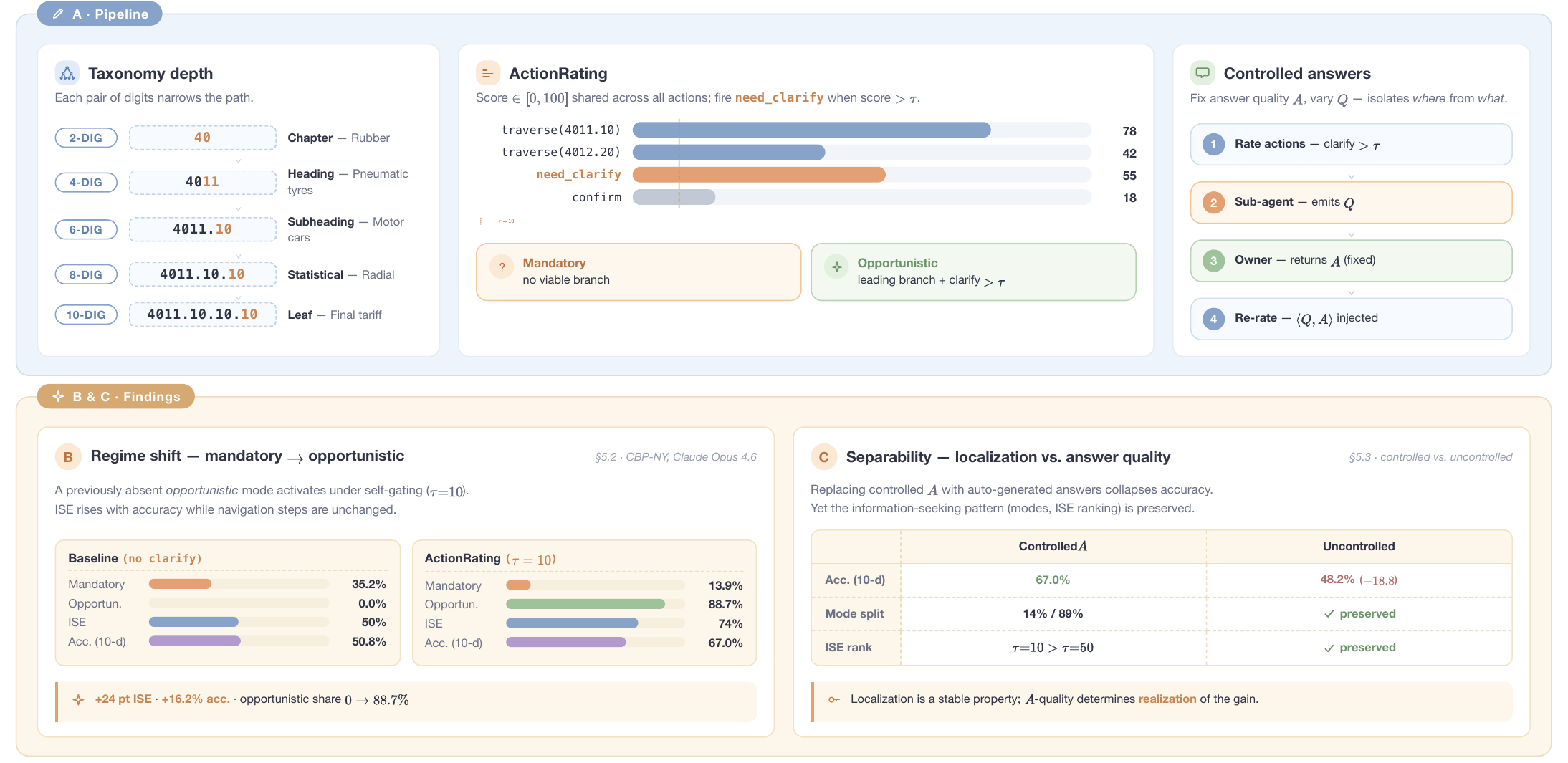}
\caption{Overview of \method{}.
\textbf{(A)~Pipeline.}
The agent rates all candidate actions, including \texttt{need\_clarify}, on a shared $[0,100]$ ordinal scale; clarification fires when its score exceeds threshold~$\tau$.
Two modes emerge: \emph{mandatory} (no viable branch) and \emph{opportunistic} (a leading branch exists but clarify still scores above $\tau$).
A controlled answer channel fixes answer quality to isolate \emph{where} help is sought from \emph{what} is received.
\textbf{(B,~C)~Findings} (\S\ref{sec:main},~\S\ref{sec:separability}; previewed here, not part of the pipeline).
At $\tau{=}10$, opportunistic share $0\to88.7\%$, ISE $50\%\to74\%$, accuracy $50.8\%\to67.0\%$~\textbf{(B)}.
Replacing controlled with auto-generated answers collapses accuracy ($-$18.8\%) yet preserves the mode split and ISE ranking~\textbf{(C)}, supporting an empirical separation between help localization and answer-source quality under controlled degradation.}
\label{fig:hero}
\end{figure*}

\begin{abstract}
In hierarchical reasoning, failures often originate at
intermediate decision points where the agent commits to a wrong
branch without recognizing that it lacks critical information.
Rather than treating clarification as an external uncertainty
trigger, we propose \method{}, a formulation that places it
inside the agent's action space on a shared ordinal scale with
navigation, so that asking competes directly with acting at every
decision point and help-seeking becomes observable at intermediate
states.
Two structurally distinct information-seeking modes emerge from the
agent's own ratings: \emph{mandatory} (no viable branch) and
\emph{opportunistic} (residual uncertainty despite a leading
candidate).
On Harmonized Tariff Schedule classification (30{,}000-node
taxonomy, three benchmarks, 9~LLMs across 4 families), we observe
a regime shift from mandatory to opportunistic clarification, with
Information-Seeking Effectiveness (ISE), a local diagnostic
defined as the fraction of help interactions followed by a correct
next navigation step (not a final-task metric), rising from 50\%
to 74\%.
Three diagnostic contrasts fail to reproduce this structure.
A separability test shows that the information-seeking pattern
(mode split, ISE ranking) persists when answer quality is degraded
($-$18.8\% accuracy), supporting an empirical separation between
\emph{where} an agent seeks help and \emph{the quality} of the help it receives.
Under the controlled answer channel, accuracy gains reach $+$16.2\%
at 10-digit; we read this as an upper bound on what better
localization could unlock, not a deployment estimate.
\end{abstract}

\section{Introduction}
\label{sec:intro}

Language agents that reason over hierarchical structures (medical
codes, legal statutes, product taxonomies) face a recurring
failure mode: once the agent commits to a wrong intermediate
branch, every subsequent step merely elaborates an error that
should have been caught
earlier~\cite{yao2023reactsynergizingreasoningacting,
yao2023treethoughtsdeliberateproblem,
shinn2023reflexionlanguageagentsverbal,
press2023measuring,
dziri2024faith}.
Final-answer accuracy tells us \emph{that} the system failed, but
not \emph{where}, at which decision point the agent lacked the
information to proceed safely.
The core question is deceptively simple: \emph{when should the
agent ask for help instead of committing?}

\paragraph{Why current approaches fall short.}
Existing designs treat clarification as external to the reasoning
trajectory: a confidence threshold~\cite{kadavath2022language}, a
prompt instruction (``ask if unsure''), or sampling-based
disagreement~\cite{kuhn2023semantic}.
These mechanisms decouple the decision to ask from the decision to
act, leaving two problems unsolved.
First, they do not make information-seeking behavior \emph{structurally
observable}: we cannot distinguish an agent that asks because no
viable branch exists from one that asks to reduce residual
uncertainty.
Second, they confound \emph{where} help was sought with
\emph{what} help was received: an agent that asks more may perform
better simply because it gets better information.

\paragraph{Clarification as action.}
We propose \method{}, a formulation that addresses both problems
by placing clarification inside the agent's own action space
(Figure~\ref{fig:hero}).
The agent scores candidate next actions, including a dedicated
clarification action, on a shared $[0, 100]$ ordinal scale, so
that asking competes directly with acting at each decision point.
This shared-scale competition makes the local need for help
observable without any external uncertainty estimator.
Two structurally distinct modes emerge from the agent's own
ratings:
\emph{mandatory} help, where clarification is top-ranked and no
navigation branch is viable, and
\emph{opportunistic} help, where a leading branch exists but a
targeted question can reduce residual uncertainty before
commitment.

\paragraph{Isolating help localization.}
To analyze information-seeking behavior cleanly, we must separate
\emph{where} help was sought from \emph{what} was received.
We pair \method{} with a controlled answer channel that fixes
answer quality, analogous to holding one experimental factor
constant to analyze another.
We also track \emph{Information-Seeking Effectiveness}
(ISE), the fraction of help interactions after which the
agent's next navigation lands on the correct path, as a local
utility probe (\S\ref{sec:ise}).
Mode shift alone shows structural change but not utility;
ISE alone measures local usefulness but not global structure;
accuracy alone is confounded by answer quality.
Together the three provide converging evidence.

\paragraph{Test bed.}
We evaluate on Harmonized Tariff Schedule (HTS) classification, a
language-mediated taxonomy of 30{,}000+ nodes where item
descriptions are free-text, taxonomy headings are natural-language
definitions, and clarification is itself a language-generation act.
HTS provides the structural prerequisites (deep branching,
repeated intermediate commitments, genuine information gaps, and
verifiable ground truth) that make the measurement question
nontrivial (\S\ref{sec:testbed}).

\paragraph{Contributions.}
\textbf{(1)~Framework.}
We formulate clarification as a selectable action competing with
navigation on a shared ordinal scale, yielding a self-gated
mechanism that makes information-seeking behavior directly
observable.
\textbf{(2)~Behavioral analysis.}
The framework reveals a regime shift, not more questions but a
structural transition from mandatory to opportunistic
clarification, with ISE rising from 50\% to 74\%.
Three diagnostic contrasts (prompt-level, sampling-based,
rating-only) do not reproduce this shift.
\textbf{(3)~Separability.}
When answer quality is degraded, accuracy collapses ($-$18.8\%)
while the information-seeking pattern (mode split, ISE ranking)
is preserved, supporting an empirical separation between help
localization and answer-source quality under controlled
degradation.
Accuracy gains under the controlled answer channel
($+$16.2\% at 10-digit) are read as an upper bound on what
better localization could unlock, not a deployment estimate.
Evaluation spans 9~LLMs (4 families), three benchmarks, component
ablation, and threshold sensitivity.

\section{Related Work}
\label{sec:related}

Our work intersects LLM agents for structured
reasoning~\cite{yao2023reactsynergizingreasoningacting,yao2023treethoughtsdeliberateproblem,shinn2023reflexionlanguageagentsverbal,zhou2024lats,schick2024toolformer,liu2023agentbench,sumers2024cognitive},
self-evaluation and
uncertainty~\cite{wang2023selfconsistencyimproveschainthought,madaan2023selfrefineiterativerefinementselffeedback,cobbe2021trainingverifierssolvemath,lightman2024letsverify,kadavath2022language,kuhn2023semantic,lin2022teaching,zheng2023judging},
information-seeking and
clarification~\cite{settles2012active,wang2023interactivenaturallanguageprocessing,Aliannejadi_2019,zamani2020generating,rao2018learning,rahmani2023survey},
selective prediction and
abstention~\cite{geifman2017selective,elyanov2010foundations,kamath2020selective},
hierarchical
classification~\cite{silla2011survey,kowsari2017hdltex,shimura-etal-2018-hft,banerjee2019hierarchical,zhou2020hierarchy,Mao_2019},
and multi-step
reasoning~\cite{wei2023chainofthoughtpromptingelicitsreasoning,zhou2023leasttomost,khot2023decomposed,gao2023pal,nye2021show,zelikman2022star,hao2023reasoning,besta2024graph,huang2024large,dua2022successive}.
A full discussion is in Appendix~\ref{app:related}.
Three distinctions position our contribution.
\emph{First}, existing agent frameworks address general reasoning
over flat or lightly structured spaces; we target deep hierarchical
taxonomies where each step narrows the search space.
\emph{Second}, self-evaluation methods rate final answers or sample
agreement; we rate candidate \emph{actions}, including
clarification, on a shared ordinal scale, so that clarification
competes directly with navigation rather than being triggered by
final-answer confidence or sampling disagreement.
\emph{Third}, prior clarification work assumes external uncertainty
estimators or human interlocutors; our mechanism is entirely
\emph{self-gated} from the agent's own action ratings.
\section{Framework}
\label{sec:framework}
\subsection{Hierarchical Navigation as MDP}
\label{sec:mdp}

We model hierarchical reasoning as an episodic Markov Decision
Process~\cite{puterman1994markov}
$\mathcal{M} = (\mathcal{S}, \mathcal{A}, \mathcal{T},
\mathcal{R})$.
\textbf{States} are taxonomy nodes augmented with the item
description and navigation history.
\textbf{Actions} comprise five types:
\texttt{traverse\_child}, \texttt{backtrack},
\texttt{need\_clarify}$(q)$, \texttt{jump}$(c)$, and
\texttt{confirm}.
\textbf{Transitions} induced by a selected taxonomy action are
deterministic (the navigation environment is a fixed tree);
stochasticity enters through the LLM policy $\pi(a \mid s)$ and
through the answer-generation channel invoked by
\texttt{need\_clarify}.
\textbf{Rewards} assign $+1$/$-1$ for correct/incorrect
classification.

\subsection{\method{}: Asking as a Selectable Action}
\label{sec:actionrating}

The core idea is to make help-needed states observable by placing
clarification inside the agent's own action space rather than
treating it as an external decision (a worked example of the
self-gated reentry cycle is in Figure~\ref{fig:framework},
Appendix~\ref{app:proofs}).
\method{} asks the agent to rate its top-$K$ candidate
actions, including a dedicated \texttt{need\_clarify}
action, on a $[0, 100]$ ordinal relevance scale before
committing (full rating prompt in Appendix~\ref{app:rating_prompt}).
At step $t$, the agent produces:
\[
  \{(a_i, s_i, r_i)\}_{i=1}^{K}, \quad a^* = \arg\max_i s_i
\]
where $a_i$ is the $i$-th candidate action, $s_i \in [0, 100]$ its
ordinal score, $r_i$ a one-sentence rationale, and $a^*$ the
selected action.
The action-rating step itself is implemented within a single navigation
call per step.
However, when clarification is triggered, the full system incurs
additional sub-agent and reentry calls at that step
(see the accuracy--cost analysis in \S\ref{sec:discussion}).

The rating serves two functions: (1) it makes help-needed states
directly observable via the ask-vs-act competition, producing the
mandatory/opportunistic distinction that is our primary analytical
object; and (2) it forces deliberative comparison between
candidates before committing.
In our experiments (Appendix~\ref{app:ablation_detail}), the behavioral change
comes primarily from self-gated help-seeking rather than from
rating alone, indicating that the rating's main value lies in
enabling observation and gating of help-needed states.

\paragraph{Controlled answer channel.}
To isolate help localization from answer quality, we use a
controlled answer channel, analogous to holding one experimental
factor fixed while analyzing another.
Two paired conditions complete the design:
(1)~a \textbf{controlled condition} that fixes answer quality high,
so behavioral differences primarily reflect help localization rather than variation in answer quality; and
(2)~a \textbf{degraded condition} that removes privileged access as
a \emph{separability test}: information-seeking patterns surviving
while accuracy collapses provides evidence for an empirical
separation between help localization and answer-source quality
(\S\ref{sec:separability}).
The controlled channel simulates a knowledgeable product owner
who can provide authoritative attribute facts (material composition,
intended use, manufacturing method) while explicit classification
codes are masked (see Appendix~\ref{app:clarify_prompt}).
A post-hoc audit confirms that 96\% of answers contain only
domain or technical-specification knowledge (\S\ref{sec:separability}).
Accuracy numbers are therefore upper bounds, not deployment
estimates.

\subsection{Self-Gated Information Seeking}
\label{sec:reentry}

A key property of the rating is that \texttt{need\_clarify} competes
directly with navigation actions on the same scale.
When \texttt{need\_clarify} appears among the top-$K$ with score
$\geq \tau$ (\emph{clarification threshold}), the agent invokes a
clarification sub-agent \emph{at the current node}:

The procedure has four stages:
\textbf{(1)~Detect}: identify $\exists\, i \leq K$ such that
$a_i = \texttt{need\_clarify} \wedge s_i \geq \tau$;
\textbf{(2)~Clarify}: invoke sub-agent
$\hat{a} = \text{ClarifyAgent}(q_i, \text{item})$;
\textbf{(3)~Inject}: add the answer to observation $o_t$;
\textbf{(4)~Re-select}: run action selection again with the
enriched observation (\emph{reentry}).

This \emph{self-gated reentry} requires no changes to the outer
navigation loop: clarification is absorbed within the step.
The threshold $\tau$ controls aggressiveness: lower values trigger
more clarifications.
We analyze sensitivity in \S\ref{sec:threshold}.
\paragraph{Two help-needed modes.}
We define two modes purely from the rank of
\texttt{need\_clarify} within the top-$K$ list.
\emph{Mandatory help}: \texttt{need\_clarify} is the
top-ranked action ($\text{rank}=1$); no navigation branch scores
above it.
\emph{Opportunistic help}: a navigation action leads the
ranking, but \texttt{need\_clarify} appears among positions
2--$K$ with score $\geq \tau$.

The definition is operational: it depends only on the rank, not on
any interpretation of why the agent ranked it there.
Empirically, rank-1 placement tends to correspond to states where
no navigation branch appears viable, while lower-ranked placement
corresponds to states with a preferred branch but residual
uncertainty, consistent with the classical notion of value of
information~\cite{howard1966information,raiffa1961applied}.
Both modes are self-triggered from the same ordinal action-rating
signal, requiring no external uncertainty estimator.
Together they form an observational taxonomy, not a classification
of true epistemic need, but a structured partition that produces a
consistent behavioral signature (mode shift, ISE improvement,
separability) absent under simpler triggers
(\S\ref{sec:experiments}).

Structural properties (monotone trigger sets, bounded reentry)
and a threshold-policy sanity check stating a sufficient
single-crossing condition under which a threshold policy is
optimal are deferred to Appendix~\ref{app:proofs}.
We treat this as a conceptual sanity check, not a guarantee for
LLM-emitted scores, which are not assumed to be calibrated.

\section{Experiments}
\label{sec:experiments}

\subsection{HTS as a Language-Mediated Testbed}
\label{sec:testbed}

HTS classification is a \emph{language-mediated} hierarchical
reasoning task: item descriptions are free-text and inherently
ambiguous (e.g., ``cough drops'' could be medicament or
confectionery), taxonomy nodes are defined by natural-language
headings with legal-text qualifications, and clarification is
itself a language-generation act; the agent must formulate a
discriminative question in natural language and interpret a textual
answer.
The entire reasoning chain, from item description through
intermediate decisions to clarification, operates in language space.

We require four structural preconditions for the measurement
question to be nontrivial:
(i)~\emph{deep branching} so that early errors compound,
(ii)~\emph{repeated intermediate commitments} so that
information-seeking varies across steps,
(iii)~\emph{information gaps} severe enough that targeted
clarification carries real value,
and (iv)~\emph{verifiable ground truth}.
Flat classification fails (i)--(ii); shallow QA benchmarks
fail (i); open-ended reasoning fails (iv).
HTS, a hierarchical taxonomy used in international trade to
assign 10-digit codes to imported goods, satisfies all four:
30{,}000+ nodes across 5 levels (branching factors up to 50+),
General Rules of Interpretation requiring special-case
protocols~(Appendix~\ref{app:mdp_validation}),
systematically incomplete product descriptions, and
verified ground-truth codes from U.S.\ Customs
rulings~\cite{cbp_cross}.
We construct a knowledge
graph~\cite{hts_data,pan2024unifying} as the MDP
environment (Appendix~\ref{app:kg}).
We separate \textbf{Layer~1} (domain
instantiation: KG, GRI protocols, answer channel; HTS-specific,
must be re-implemented per domain) from \textbf{Layer~2}
(measurement protocol: action-space formulation,
mandatory/opportunistic analysis, ISE, threshold
sweep; portable to any tree-structured reasoning task).

\paragraph{Datasets and metrics.}
We evaluate on three benchmarks:
\textbf{CBP-NY} (1{,}181 products extracted from public CBP rulings~\citep{cbp_cross} via LLM-based pipeline (Appendix~\ref{app:extraction_prompt}),
\textbf{ATLAS} (200 samples~\cite{yuvraj2025atlasbenchmarkingadaptingllms}),
and \textbf{HSCodeComp} (632 expert-annotated
records~\cite{yang2025hscodecomprealisticexpertlevelbenchmark}).
We report hierarchical accuracy at each depth level (2--10 digits),
success rate, average navigation steps, and ISE (\S\ref{sec:ise}).

\subsection{Setup}
\label{sec:baselines}

We evaluate 9~LLMs across 4 families as MDP navigators; \method{}
($\tau{=}10$) is applied to 5 of them.
The \textbf{baseline} uses greedy action selection without scoring
or gating (Appendix~\ref{app:nav_prompt}).
The controlled answer channel uses CBP ruling
attributes with codes masked
(Appendix~\ref{app:clarify_prompt},~\ref{app:oracle_ablation}).
Two \emph{diagnostic contrasts} test alternative explanations:
\textbf{H1}: a prompt-level instruction suffices
(CoT-Ask-if-Unsure; Appendix~\ref{app:cot});
\textbf{H2}: a sampling-based trigger suffices
(Self-Consistency, $N{=}3$;~\cite{wang2023selfconsistencyimproveschainthought}; Appendix~\ref{app:sc}).
\textbf{H3} (deliberation without actioning) is tested by the
rating-only ablation ($\tau{=}101$).

\section{Results}
\label{sec:results}

\subsection{Information-Seeking Behavior under \method{}}
\label{sec:main}

\begin{table*}[t]
\centering
\caption{
  Hierarchical accuracy (\%) across models and confidence-triggering methods on HTS
  classification ($N{=}1{,}181$).
  Cell shading: \colorbox{green!25}{high} to \colorbox{red!10}{low}.
  \textit{CoT-Ask-if-Unsure}: single prompt instruction to ask clarification when uncertain.
  \textit{Self-Consistency} ($N{=}3$): trigger clarification on action disagreement
  ($\dagger$ ${\approx}19$ LLM calls/record).
  \textbf{AR ($\tau{=}10$)}: Claude Opus 4.6 with \method{}.
  Significance markers on $\Delta$ rows are based on non-parametric paired bootstrap
  ($n_{\mathrm{boot}}{=}5{,}000$):
  $^{***}$ 95\,\% CI strictly positive;
  $^{\mathrm{ns}}$ CI includes zero.
}
\label{tab:main}
\small
\resizebox{\textwidth}{!}{%
\begin{tabular}{llccccccc}
\toprule
\textbf{Model} & \textbf{\# Params} & \textbf{Succ.\,(\%)} & \textbf{2-digit} & \textbf{4-digit} & \textbf{6-digit} & \textbf{8-digit} & \textbf{10-digit} & \textbf{Avg Steps} \\
\midrule
\multicolumn{9}{l}{\textit{Closed-Source Models (Baseline MDP)}} \\
\midrule
    Claude Opus 4.6     & --- & 97.3 & \cellcolor{green!25}79.3 & \cellcolor{green!25}70.9 & \cellcolor{green!15}61.8 & \cellcolor{yellow!20}54.4 & \cellcolor{yellow!20}50.8 & 5.6 \\
    Claude Sonnet 4.5   & --- & 95.1 & \cellcolor{green!25}77.6 & \cellcolor{green!15}67.2 & \cellcolor{yellow!20}56.4 & \cellcolor{yellow!20}50.5 & \cellcolor{orange!15}47.2 & 6.7 \\
    Claude Haiku 4.5    & --- & 96.5 & \cellcolor{green!25}72.2 & \cellcolor{green!15}61.3 & \cellcolor{orange!15}49.8 & \cellcolor{orange!15}42.5 & \cellcolor{red!10}37.8 & 6.5 \\
\midrule
\multicolumn{9}{l}{\textit{Open-Source Models (Baseline MDP)}} \\
\midrule
    Kimi K2             & 1T   & 97.6 & \cellcolor{green!25}76.8 & \cellcolor{green!15}67.6 & \cellcolor{yellow!20}56.8 & \cellcolor{orange!15}48.6 & \cellcolor{orange!15}44.1 & 5.7 \\
    DeepSeek V3         & 671B & 90.0 & \cellcolor{green!25}75.5 & \cellcolor{green!15}65.2 & \cellcolor{yellow!20}53.7 & \cellcolor{orange!15}45.9 & \cellcolor{orange!15}41.4 & 6.2 \\
    Mistral Large 3     & 123B & 96.7 & \cellcolor{green!25}72.2 & \cellcolor{green!15}60.8 & \cellcolor{orange!15}49.9 & \cellcolor{orange!15}42.3 & \cellcolor{red!10}38.8 & 5.9 \\
    GPT-OSS 120B        & 120B & 96.5 & \cellcolor{green!25}72.9 & \cellcolor{green!15}61.3 & \cellcolor{orange!15}49.9 & \cellcolor{orange!15}41.3 & \cellcolor{red!10}36.9 & 6.0 \\
    Minimax M2          & 230B & 96.8 & \cellcolor{green!15}67.6 & \cellcolor{yellow!20}55.5 & \cellcolor{orange!15}45.5 & \cellcolor{red!10}38.8 & \cellcolor{red!10}35.4 & 6.1 \\
    Qwen3 235B          & 235B & 93.6 & \cellcolor{green!15}65.9 & \cellcolor{yellow!20}54.7 & \cellcolor{orange!15}44.2 & \cellcolor{red!10}36.6 & \cellcolor{red!10}32.9 & 6.9 \\
\midrule
\multicolumn{9}{l}{\textit{Simpler Confidence-Triggering Alternatives (Claude Opus 4.6)}} \\
\midrule
    + CoT-Ask-if-Unsure                    & --- & 97.5 & \cellcolor{green!35}80.6 & \cellcolor{green!25}71.5 & \cellcolor{green!15}63.0 & \cellcolor{yellow!20}55.6 & \cellcolor{yellow!20}52.0 & 5.5 \\
    + Self-Consistency ($N{=}3$)$^\dagger$ & --- & 96.4 & \cellcolor{green!35}85.3 & \cellcolor{green!25}77.4 & \cellcolor{green!15}68.9 & \cellcolor{green!15}62.9 & \cellcolor{yellow!20}59.5 & 6.1 \\
\midrule
    Claude Opus 4.6 + \method{} ($\tau{=}10$) & --- & 97.8 & \cellcolor{green!35}\textbf{87.2} & \cellcolor{green!35}\textbf{82.0} & \cellcolor{green!25}\textbf{75.2} & \cellcolor{green!15}\textbf{69.5} & \cellcolor{green!15}\textbf{67.0} & 5.4 \\
    $\Delta$ vs.\ best baseline\textsuperscript{***} & & & \textbf{+7.9} & \textbf{+11.1} & \textbf{+13.4} & \textbf{+15.1} & \textbf{+16.2} & \\
    $\Delta$ vs.\ Self-Consistency & & & $+$\textbf{1.9}$^{\mathrm{ns}}$ & $+$\textbf{4.6}\textsuperscript{***} & $+$\textbf{6.3}\textsuperscript{***} & $+$\textbf{6.6}\textsuperscript{***} & $+$\textbf{7.5}\textsuperscript{***} & \\
\bottomrule
\end{tabular}%
}
\end{table*}

Table~\ref{tab:main} presents results across all 9
baseline models, two diagnostic contrasts
(CoT-Ask-if-Unsure and Self-Consistency, $N{=}3$),
and Claude Opus 4.6 with \method{}.
The primary signals are \emph{what changes in
information-seeking behavior}, not the accuracy numbers themselves
(which are upper bounds under the controlled answer channel).

\textbf{A distinctive clarification pattern emerges.}
\method{} ($\tau{=}10$) produces a three-part behavioral signature
absent from all comparison conditions:
(i)~a regime shift from mandatory to opportunistic
clarification
(35.2\%$\to$13.9\% mandatory; 0\%$\to$88.7\% opportunistic),
(ii)~rising local utility (ISE: 50\%$\to$74\%), and
(iii)~accuracy co-movement at every hierarchy depth.
This is a structural change in \emph{where} the agent seeks
clarification, not merely \emph{how often}; the volume increase
is a consequence of opportunistic mode activation, not its cause
(\S\ref{sec:ise}--\ref{sec:separability}).

\textbf{Gating, not scoring, drives the regime shift.}
Rating-only ($\tau{=}101$; Appendix~\ref{app:ablation_detail})
retains the full scoring apparatus but disables the gating
threshold so that no clarification ever fires.
Despite having the deliberation step, it yields $-$0.9\% and
produces no opportunistic events, directly isolating
\emph{gating} as the mechanism behind the regime shift.
CoT-Ask-if-Unsure (H1) and Self-Consistency (H2) similarly
fail to reproduce the three-part signature, arguing against
prompt-level instruction and sampling disagreement as
sufficient triggers.
Among the tested alternatives, only shared-scale action
competition, where asking competes with acting at the same local
decision point, produces the observed structure.

\textbf{Upper-bound accuracy under controlled answers.}
Under the controlled answer channel, Claude Opus 4.6 reaches 67.0\%
10-digit accuracy ($+$16.2\% over baseline 50.8\%), with gains
increasing monotonically with depth, consistent with the claim
that deeper hierarchy levels benefit most from better
help localization.
These numbers bound what better localization
\emph{could} unlock when high-quality answers are available.
Cost rises from 6.0 to 10.4 calls/record
(Figure~\ref{fig:acc_vs_cost}).
The behavioral signature generalizes across 4~LLM families
(Table~\ref{tab:multimodel},
Appendix~\ref{app:multimodel_detail}) and two additional
benchmarks (ATLAS and HSCodeComp) with $\tau{=}10$ locked
from CBP-NY, providing an out-of-sample generalization test
(Table~\ref{tab:extbenchmarkmodel},
Appendix~\ref{app:crossbench_detail}).

\subsection{Information-Seeking Effectiveness (ISE)}
\label{sec:ise}

A central question is whether self-gating merely increases the
\emph{volume} of help-seeking or also improves its local
usefulness; i.e., whether help was requested at states that
genuinely needed it.
We define Information-Seeking Effectiveness (ISE) as a
\emph{localized help-utility probe}: the fraction of help
interactions after which the agent's next navigation action
lands on the correct path:
\[
  \text{ISE} = \frac{\text{\# QA followed by correct traverse}}
                     {\text{\# total QA interactions}}
\]
ISE tests whether each identified help interaction was
\emph{locally useful}: if help was requested and the next action
was correct, the interaction was productive at the one-step level.
Because the controlled answer channel fixes answer quality, ISE
variation across conditions reflects differences in \emph{where}
the agent sought help, not \emph{what} it received.

\textbf{Self-gating improves quality, not just volume.}
At $\tau{=}10$, the agent issues 6$\times$ more clarifications
(2.4 vs.\ 0.4/record), yet ISE rises from 50\% to 74\%
(Figure~\ref{fig:ise}a).
Virtually all additional volume is opportunistic, and both modes
attain $\approx$75\% ISE (Figure~\ref{fig:ise}b).
At $\tau{=}1$, volume increases further (5.9~QA/record) but ISE
\emph{drops} to 62\%: the additional questions are less likely to
land on the correct path.
This dissociation between volume and quality is the clearest
evidence that the framework distinguishes \emph{asking at the right
place} from \emph{asking more often}.

\textbf{Opportunistic clarification corrects misaligned trajectories.}
Table~\ref{tab:ise_counterfactual} further dissects opportunistic ISE
by the agent's traversal state at clarification time.
Even when the agent was already on the correct HTS path, ISE is
83.6\%, suggesting that on-path clarifications refine rather than
derail correct trajectories.
More tellingly, when the agent was \emph{off} the correct path, where
clarification must actively steer the agent back, ISE remains 67.3\%.
This argues against the confound that high aggregate ISE is driven
purely by easy, already-correct cases: the clarification sub-agent
appears to meaningfully correct misaligned trajectories.

\begin{figure*}[t]
\centering
\includegraphics[width=0.82\textwidth]{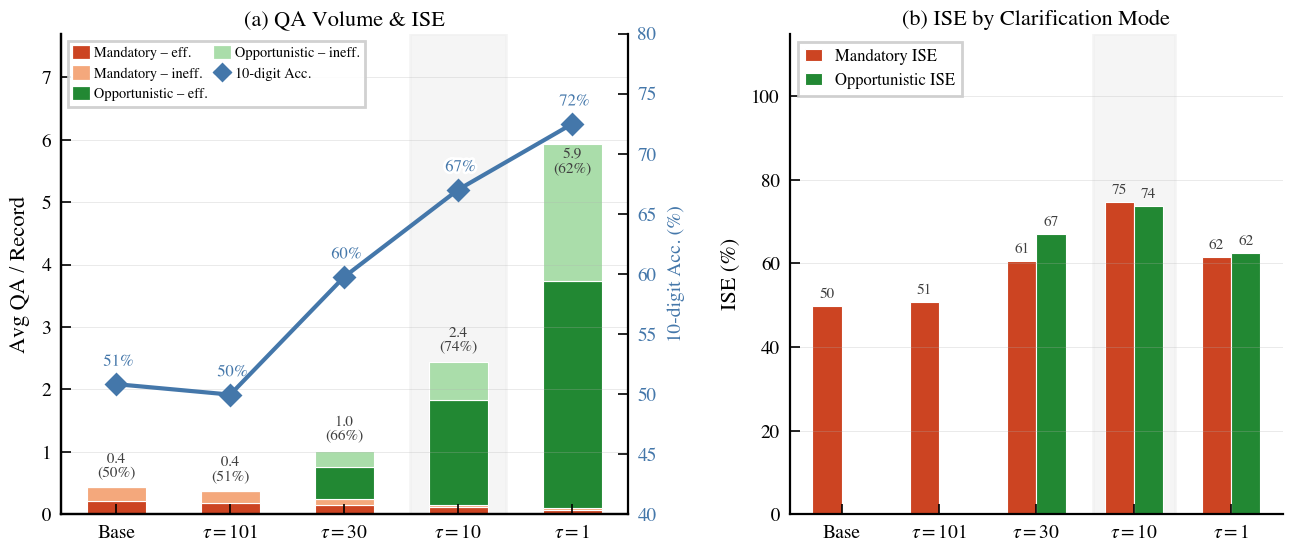}
\caption{
  \textbf{Clarification behavior on CBP-NY (n=1{,}181), Claude Opus 4.6.}
  \textbf{(a) QA Volume \& ISE.}
  Stacked bars: average QA per record, split by mode
  (orange: mandatory; green: opportunistic) and outcome
  (solid: effective; light: ineffective). Right axis: 10-digit
  accuracy (blue diamonds). Annotations above each bar: total
  QA/record and overall ISE~(\%). The $\tau{=}10$ condition
  (outlined) shifts volume from mandatory to opportunistic while
  maximising accuracy.
  \textbf{(b) ISE by Clarification Mode.}
  Grouped bars compare mandatory (orange) vs.\ opportunistic
  (green) ISE at each threshold; at $\tau{=}10$, both reach
  $\approx$75\%, while at $\tau{=}1$ both drop to 62\%
  (over-triggering).
}
\label{fig:ise}
\end{figure*}

\subsection{Threshold Sensitivity}
\label{sec:threshold}

We do not tune $\tau$ on target benchmarks; instead we use
threshold sweeps to characterize behavior on CBP-NY and lock
$\tau{=}10$ before any ATLAS or HSCodeComp evaluation. Treating
$\tau$ as a \emph{behavioral phase control} in this analysis,
sweeping it maps out a phase diagram of help-seeking structure
(Table~\ref{tab:threshold}, Appendix~\ref{app:threshold_detail}).
The $\tau{=}50{\to}30$ transition marks the clearest phase
boundary (opportunistic rate: 9.7\%$\to$51.9\%; accuracy:
51.2\%$\to$59.8\%), corresponding to the onset of widespread
opportunistic mode activation.
At $\tau{=}10$, 90.9\% of records involve help-seeking (76.8\%
opportunistic), achieving 74\% ISE at only 2.4~QA/record.
Navigation steps are unchanged across all settings (5.4--5.7),
indicating that gating alters \emph{where} the agent seeks help,
not \emph{how} it navigates.
The sweep was conducted on CBP-NY only; $\tau{=}10$ was locked
before any ATLAS or HSCodeComp evaluation
(selection protocol in Appendix~\ref{app:threshold_detail}).

\subsection{Separability: Help Localization vs.\ Answer Quality}
\label{sec:separability}

The regime shift and ISE improvement above are observed under
controlled answers.
The critical test is whether these patterns reflect the agent's
information-seeking ability or are artifacts of answer quality.
A separability test (Table~\ref{tab:oracle_ablation} in
Appendix~\ref{app:oracle_ablation}) replaces the controlled
channel with fully-automated answers derived from the product
description alone.

\textbf{Accuracy collapses; information-seeking pattern is preserved.}
\method{} loses nearly all accuracy gain at 10-digit
(67.0\%~$\to$~48.2\%) while the baseline is unaffected
(50.8\%~$\to$~49.2\%): a $-$18.8\% gap attributable to answer
quality.
Crucially, the information-seeking pattern
itself (mandatory/opportunistic split, ISE ranking across
thresholds) is preserved even when answer quality is degraded.

\textbf{Interpretation.}
The agent can locate states where reasoning needs help even when
the answer source is weak; what it cannot do without a
knowledgeable respondent is \emph{realize} the downstream accuracy
benefit.
This dissociation between localization and realization
provides evidence that the two can be analyzed as separable
factors under controlled degradation, with localization behavior
appearing more stable than realized accuracy across answer-quality
conditions.
The behavioral signature additionally satisfies four validity
criteria:
\emph{stability} (replicates across 4~LLM families and
3~benchmarks), \emph{interpretability} (mandatory vs.\
opportunistic modes have clear semantic content),
\emph{contrastiveness} (three diagnostic contrasts fail to produce
the same structure), and \emph{predictive local utility} (ISE
indicates that identified help states are productive at the
one-step level).

\textbf{Knowledge-channel audit.}
We audit all 2{,}875 Q/A pairs from the CBP-NY run with action
rating ($\tau{=}10$) along two axes: six question categories
(e.g., Material/Composition, Function/Use) and five answer types
(Product Attribute, PA-Essential Character,
Classification Criteria~(CC), Unavailable, Deflected).
Table~\ref{tab:qa_crosstab} shows the full cross-tabulation: 80\%
of answers are plain product attributes, and only 3.5\% of questions
(102 pairs) explicitly name an HTS chapter, heading, or note.
Table~\ref{tab:qa_crosstab_hts} isolates this HTS-referencing
subset: the guardrail deflects 34\% outright; 37\% yield plain
product-attribute answers; and only 23 pairs (0.8\% of all 2{,}875)
reach a \emph{Classification Criteria} answer (i.e., the oracle
directly affirms or denies a named chapter note or heading criterion,
such as ``yes, it qualifies as rubber under Chapter~40, Note~1'').
Table~\ref{tab:post_qa_navigation} then shows that even those 23
CC answers produce \emph{lower} navigation success than plain
product-attribute answers (ISE 62.5\% vs.\ 76.2\%), arguing
against oracle leakage as the main driver of the accuracy gain
(Q\&A examples and per-level breakdown in
Appendix~\ref{app:knowledge_audit}).

Trajectory-level analysis (Appendix~\ref{app:trajectory_detail})
shows that the score gap between the top-ranked navigation
action and clarification narrows at deeper hierarchy levels,
consistent with increasing local ambiguity at finer
classification granularity, while navigation steps and
backtracking rates remain comparable to the baseline.

\subsection{Qualitative Clarification Behavior}
\label{sec:qualitative}

To illustrate what the two modes look like in practice, we present
representative examples from the CBP-NY evaluation.

\paragraph{Mandatory clarification.}
For \emph{oval-shaped sugar confectionery cough drops containing
10\,mg menthol}, the baseline agent commits to pharmaceuticals
based on the dosage language (``10\,mg per dose'').
Under \method{}, no navigation branch scores above $\tau$ at the
Lv.2 node; \texttt{need\_clarify} is top-ranked (score 68, next
branch 31).
The agent asks: ``Is this product put up in measured doses or for
retail sale as a medicament?''
The answer (``No, this is sugar confectionery for retail
consumption'') resolves the ambiguity, and the re-scored
navigation correctly routes to confectionery.

\paragraph{Opportunistic clarification.}
For a \emph{granular copolymer: 69\% butadiene, 20\% methyl
methacrylate, 9\% methacrylic acid, 2\% divinylbenzene}, the
agent's top-ranked navigation action at Lv.3 is ``polymers of
olefins'' (score 72), but \texttt{need\_clarify} appears at rank~2
(score 48, above $\tau{=}10$).
The question (``Is butadiene considered an olefin for
classification purposes?'') targets a domain convention:
chemically, butadiene is a conjugated diene, but classification
convention treats butadiene polymers as olefin polymers.
The controlled answer confirms the convention; without it, the
automated system applies strict chemistry and misroutes to
``other resins.''

\paragraph{Pattern.}
Mandatory questions tend to target \emph{missing essential
attributes} (composition, primary use, form) at early levels where
no branch is preferred, while opportunistic questions target
\emph{fine-grained disambiguation} (domain conventions, threshold
values) at deeper levels where a leading branch exists.
The two modes capture structurally different information needs,
not merely different confidence levels.

\section{Discussion}
\label{sec:discussion}

\begin{figure}[t]
\centering
\includegraphics[width=\columnwidth]{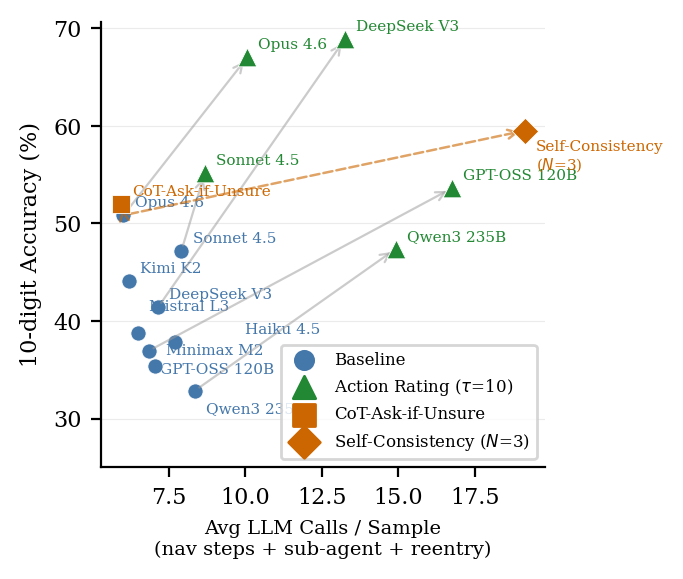}
\caption{10-digit accuracy vs.\ inference cost (LLM calls/sample).
  Every model improves with \method{} (gray arrows); simpler
  triggers (orange) do not reproduce the regime shift.
  Accuracy values are upper bounds under the controlled answer channel (product-owner simulation).}
\label{fig:acc_vs_cost}
\end{figure}

Figure~\ref{fig:acc_vs_cost} shows that the behavioral signature
is model-agnostic and that simpler triggers lie below the \method{}
Pareto frontier.
We highlight three implications.

\paragraph{Portability.}
The two-layer architecture separates domain instantiation
(Layer~1: knowledge graph, classification protocols, answer
channel) from the analysis protocol (Layer~2: shared ordinal
scale, threshold policy, ISE, separability test).
Layer~2 requires only a tree-structured action space with
potential information gaps; candidate domains include
medical coding (ICD-10), product classification (CPC), and
legal statute navigation.

\paragraph{Cost--accuracy trade-off.}
\method{} increases inference cost from 6.0 to 10.4 LLM calls per
record (73\% overhead), with the increase sublinear in accuracy
gain. Self-Consistency at $N{=}3$ achieves $+$8.7\% at 19 calls
(3.2$\times$ cost), placing it below the \method{} Pareto frontier
(Figure~\ref{fig:acc_vs_cost}).

\paragraph{Toward realistic answer sources.}
The controlled answer channel is a clean measurement baseline.
A natural next step is an \emph{answer-source ladder} varying
answer quality from the controlled channel through retrieval and
weaker LLMs to noisy human-like responses, to map how localization
benefits degrade as answer quality decreases.
The separability result (\S\ref{sec:separability}) predicts that
information-seeking \emph{patterns} remain stable across this
ladder even as accuracy varies.

\paragraph{Decomposing help-seeking.}
Separability motivates decomposing help-seeking into three
factors: localization (where to ask), question quality (what to
ask), and answer-source quality (who answers). This paper isolates
the first.

\section{Conclusion}
\label{sec:conclusion}

\method{} places clarification inside a language agent's action
space on a shared ordinal scale with navigation, yielding a
self-gated mechanism that requires no external uncertainty estimator.
We frame the contribution as a measurement protocol for
self-gated information-seeking behavior under a controlled answer
channel, not a deployment system; reported accuracy gains are
upper bounds.
The framework reveals a regime shift from mandatory to
opportunistic information-seeking (ISE: 50\%$\to$74\%), stable
across four LLM families, three benchmarks, and a range of
thresholds, with the two modes serving distinct linguistic
functions.
Three diagnostic contrasts fail to reproduce this structure, and
a separability test ($-$18.8\% accuracy under degraded answers)
supports an empirical separation between help localization and
answer-source quality. Generalization beyond HTS and evaluation
with realistic answer sources remain open directions.
\section*{Limitations}

\noindent\textbf{Controlled answer channel, not deployment.}
Accuracy numbers are upper bounds under a controlled answer
channel; deployment with realistic information sources would yield
lower gains.
We do not yet systematically vary answer quality across
intermediate regimes.

\textbf{Single domain.}
Evaluation covers three independent HTS benchmarks but
generalization to structurally distinct taxonomies (ICD-10, CPC,
legal statutes) requires re-implementing the domain layer.

\textbf{Action scores are not calibrated.}
\method{} assumes the LLM produces meaningful ordinal scores but
does not claim calibrated confidence
estimation~\cite{guo2017calibration,xiong2024can};
$\tau$ may require re-tuning across models.

\textbf{Observational taxonomy.}
The mandatory/opportunistic distinction is derived from the
agent's own ratings, not from ground-truth epistemic states.

\textbf{Practical constraints.}
Opportunistic mode issues multiple inline QA rounds per step;
latency may offset accuracy gains.
Evaluation uses English-language descriptions only.

\section*{Ethics Statement}

HTS classification has direct financial and regulatory implications:
incorrect codes can lead to improper duty assessment, and automation
errors may have downstream financial or regulatory consequences.
Our system is intended as a decision-support tool and should be
validated by domain experts before deployment.
The benchmark data is derived from publicly available CBP rulings
and does not contain personally identifiable information.


\bibliography{anthology,custom}

\newpage
\appendix

\section*{Appendix Roadmap}
\noindent The appendix is organized to support specific reviewer
questions rather than to extend the main argument.
\textbf{App.\ \ref{app:proofs}} provides formal sanity checks
(threshold-policy single-crossing condition, bounded reentry) and
a worked example of the self-gated cycle.
\textbf{App.\ \ref{app:kg}} documents the knowledge graph and
classification protocols underlying CBP-NY.
\textbf{App.\ \ref{app:oracle_ablation}--\ref{app:knowledge_audit}}
audit the controlled answer channel and argue against oracle
leakage as the primary driver of accuracy.
\textbf{App.\ \ref{app:mdp_validation}--\ref{app:sc}} cover MDP
component ablations and the diagnostic-contrast baselines
(CoT-Ask-if-Unsure, Self-Consistency).
\textbf{App.\ \ref{app:multimodel_detail}--\ref{app:crossbench_detail}}
report multi-model and cross-benchmark transfer under the locked
$\tau{=}10$ protocol.
\textbf{App.\ \ref{app:threshold_detail}--\ref{app:trajectory_detail}}
present threshold sensitivity and trajectory-level mechanism
analysis.
\textbf{App.\ \ref{app:extraction_prompt}--\ref{app:rating_prompt}}
release dataset-construction and full prompt templates for
reproducibility.

\section{Formal Properties and Proofs}
\label{app:proofs}

\subsection{Worked Example of the Self-Gated Reentry Cycle}
\label{app:reentry_example}

Figure~\ref{fig:framework} traces a single record through the
self-gated reentry cycle on a CBP-NY example, illustrating how
clarification competes with navigation on a shared ordinal scale
rather than being externally triggered.

\paragraph{Top-down traversal.}
The agent enters the root of the HTS tree and scores every
candidate action---each child branch as well as
\texttt{need\_clarify}---on the shared $[0,100]$ scale.
At the 2-digit and 4-digit levels (left panel), one navigation
branch dominates, so the agent commits without invoking the
clarification action. This corresponds to the standard
hierarchical decoding regime: the rating layer is active at every
step, but the gating condition is not met.

\paragraph{Reentry loop at the 8-digit node.}
Descent stalls at the 8-digit node (dashed box), where no single
branch exceeds the others by a clear margin and \texttt{clarify}
rises to score~45, above $\tau{=}10$. This triggers the reentry
loop: a clarification sub-agent generates a question, queries the
controlled answer channel, and the resulting $\langle Q, A\rangle$
pair is appended to the observation. The agent then re-scores
all actions at the same node with the augmented context (right
panel, Round~2). The clarification action drops, and the leading
\texttt{traverse} action now dominates at $\geq 92$, allowing the
agent to commit and continue toward the 10-digit leaf.

\paragraph{Why this matters for measurement.}
The two-round score trace exposes the behavioral object the rest
of the paper measures: a state where clarification temporarily
out-competes navigation, the agent acts on that signal, and a
re-scored navigation step follows. Mode taxonomy
(mandatory vs.\ opportunistic) is read from the rank structure of
the first round; ISE is read from whether the post-injection
action is correct. Without the shared-scale rating, none of these
quantities is observable from the trace alone.

\begin{figure*}[!t]
\centering
\includegraphics[width=0.82\textwidth]{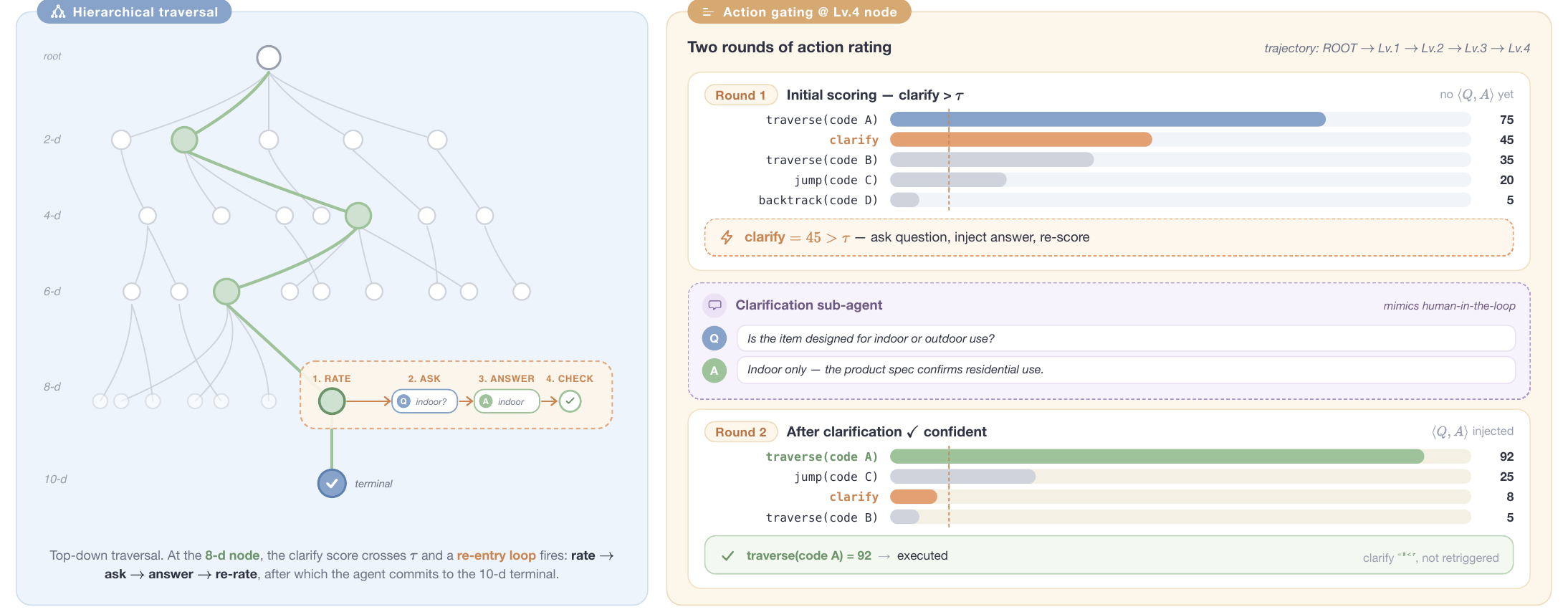}
\caption{Self-gated clarification cycle within hierarchical taxonomy navigation.
\textbf{Left:} Top-down traversal of the taxonomy tree.
At each internal node the agent scores all candidate actions, including \texttt{need\_clarify}, on a shared $[0,100]$ ordinal scale.
When the clarification score exceeds threshold~$\tau$ at the 8-digit node (dashed box), the agent enters a \emph{reentry loop}: a sub-agent resolves the query and injects $\langle Q, A \rangle$ before re-scoring.
\textbf{Right:} Two rounds of action rating. Round~1: \texttt{clarify} scores 45 ($>\tau$), triggering the sub-agent. Round~2: after answer injection the leading \texttt{traverse} action dominates ($\geq 92$); the episode terminates at the 10-digit leaf via \texttt{confirm}.}
\label{fig:framework}
\end{figure*}

\subsection{Threshold-Policy Sanity Check (Theorem~\ref{thm:threshold})}

\paragraph{Scope.}
This subsection states a sufficient condition under which a
threshold policy on $q(s)$ is optimal within the threshold family.
We do \emph{not} assume LLM-emitted scores satisfy this condition;
the result is a conceptual sanity check rather than an empirical
guarantee.

\paragraph{Setup.}
Let $\mathcal{S}$ be the set of decision states encountered by the
navigator, and let $q(s) \in \mathbb{R}$ be the agent's rating
assigned to the clarification action at state $s$.
Define the \emph{net downstream gain of clarification}:
\[
  \Delta(s) := V^{\mathrm{ask}}(s) - V^{\mathrm{act}}(s),
\]
where $V^{\mathrm{ask}}(s)$ is the expected downstream return from
clarifying before acting, and $V^{\mathrm{act}}(s)$ is the expected
downstream return from acting immediately.
For a threshold $\tau$, define the clarification policy
$\pi_\tau(s) = \mathbf{1}\{q(s) \ge \tau\}$,
with expected utility
$U(\tau) := \mathbb{E}\!\bigl[\Delta(s)\,\mathbf{1}\{q(s) \ge \tau\}\bigr]$.

\begin{proposition}[Monotonicity of trigger sets]
\label{prop:monotone}
For $A_\tau := \{s : q(s) \ge \tau\}$ and any $\tau_1 < \tau_2$,
we have $A_{\tau_2} \subseteq A_{\tau_1}$.
Hence lowering the threshold can only add clarification-triggered
states, and any nonnegative clarification-cost functional is weakly
decreasing in $\tau$.
\end{proposition}

\noindent\textbf{Assumption A} (Single-crossing conditional gain).
Let $m(t) := \mathbb{E}[\Delta(s) \mid q(s){=}t]$.
Assume $m$ is integrable and there exists $\tau^\star$ such that
$m(t) \le 0$ for $t < \tau^\star$ and
$m(t) \ge 0$ for $t \ge \tau^\star$.

\begin{theorem}[Optimal threshold under single crossing]
\label{thm:threshold}
Under Assumption~A, the threshold $\tau^\star$ maximizes
$U(\tau)$ over the threshold family $\{\pi_\tau\}_\tau$.
\end{theorem}

\subsection{Bounded Reentry}

\paragraph{Bounded reentry (stated in \S\ref{sec:reentry}).}
Each node $v$ permits at most $C$ clarifications (duplicate guard).
Let $\mathcal{V}_{\text{ep}}$ be the set of nodes visited during
the episode.
The total number of clarification events is bounded by
$N_{\text{clarify}} \le C \cdot |\mathcal{V}_{\text{ep}}|$.
Each clarification incurs exactly 2 additional LLM calls
(one sub-agent call + one reentry re-selection), so clarification
adds at most $2 N_{\text{clarify}}$ calls.
Navigation actions are bounded by $H$ (the episode step limit).
Therefore the total call count satisfies
$N_{\text{total}} \le H + 2C|\mathcal{V}_{\text{ep}}|$,
which is finite since $|\mathcal{V}_{\text{ep}}| \le H$ (each
navigation step visits at most one new node).
In our implementation, $C = 2$ and $H = 20$, giving
$N_{\text{total}} \le 20 + 4 \cdot 20 = 100$.

\begin{proof}[Proof of Proposition~\ref{prop:monotone}]
By definition, $A_\tau = \{s : q(s) \ge \tau\}$.
For $\tau_1 < \tau_2$, if $s \in A_{\tau_2}$ then
$q(s) \ge \tau_2 > \tau_1$, so $s \in A_{\tau_1}$.
Therefore $A_{\tau_2} \subseteq A_{\tau_1}$.
For any nonnegative $c(s) \ge 0$, the inclusion implies
$\mathbf{1}\{q(s) \ge \tau_2\} \le \mathbf{1}\{q(s) \ge \tau_1\}$
for all $s$, so
$\mathbb{E}[c(s)\,\mathbf{1}\{q(s) \ge \tau_2\}]
\le \mathbb{E}[c(s)\,\mathbf{1}\{q(s) \ge \tau_1\}]$.
\end{proof}

\begin{proof}[Proof of Theorem~\ref{thm:threshold}]
Let $X = q(s)$.
By the law of iterated expectation,
\begin{align*}
  U(\tau)
  &= \mathbb{E}\!\bigl[\Delta(s)\,\mathbf{1}\{X \ge \tau\}\bigr]\\
  &= \mathbb{E}\!\bigl[m(X)\,\mathbf{1}\{X \ge \tau\}\bigr],
\end{align*}
where $m(t) := \mathbb{E}[\Delta(s) \mid q(s){=}t]$.

\emph{Case 1:} $\tau < \tau^\star$.
Then
\[
  U(\tau^\star) {-} U(\tau)
  = -\mathbb{E}\!\bigl[m(X)\,\mathbf{1}\{\tau {\le} X {<} \tau^\star\}\bigr]
  \ge 0,
\]
because $m(X) \le 0$ on $\{X < \tau^\star\}$ by Assumption~A.

\emph{Case 2:} $\tau > \tau^\star$.
Then
\[
  U(\tau^\star) {-} U(\tau)
  = \mathbb{E}\!\bigl[m(X)\,\mathbf{1}\{\tau^\star {\le} X {<} \tau\}\bigr]
  \ge 0,
\]
because $m(X) \ge 0$ on $\{X \ge \tau^\star\}$ by Assumption~A.

Thus $U(\tau^\star) \ge U(\tau)$ for every $\tau$, establishing
optimality.
\end{proof}

\begin{corollary}[Selective clarification can outperform both
extremes]
\label{cor:dominance}
Under Assumption~A, if there exist both positive-gain and
negative-gain clarification states with nonzero probability
(i.e.\ $\mathbb{P}(\Delta(s) > 0,\; q(s) \ge \tau^\star) > 0$ and
$\mathbb{P}(\Delta(s) < 0,\; q(s) < \tau^\star) > 0$),
then
$U(\tau^\star) > U(+\infty) = 0$ and
$U(\tau^\star) > U(-\infty) = \mathbb{E}[\Delta(s)]$.
\end{corollary}

\begin{proof}
The no-clarification policy gives $U(+\infty) = 0$.
Since $m(X) \ge 0$ on $\{X \ge \tau^\star\}$ with strict positivity
on a subset of positive measure,
$U(\tau^\star) = \mathbb{E}[m(X)\,\mathbf{1}\{X \ge \tau^\star\}] > 0$.
The always-clarify policy gives
$U(-\infty) = \mathbb{E}[\Delta(s)]$. Then
\[
  U(-\infty) - U(\tau^\star)
  = \mathbb{E}\!\bigl[m(X)\,\mathbf{1}\{X {<} \tau^\star\}\bigr].
\]
Because $m(X) \le 0$ on $\{X < \tau^\star\}$ with strict negativity
on a subset of positive measure, the right-hand side is strictly
negative, so $U(\tau^\star) > U(-\infty)$.
\end{proof}

\noindent\textbf{Interpretation.}
This corollary formalizes a simple intuition: some states benefit
from clarification, while others do not.
A policy that asks everywhere pays unnecessary clarification cost,
whereas a policy that never asks forgoes high-value interventions.
A threshold policy can improve over both by selectively retaining
only those states whose expected clarification gain is nonnegative.

\section{Knowledge Graph Construction}
\label{app:kg}

We represent the HTS as an augmented directed graph
$\mathcal{G} = (V,\, E_T \cup E_R)$ constructed from the official
USITC HTS 2025 Revision~26 data~\cite{hts_data}, where $|V|=30{,}202$
nodes span five hierarchical levels (chapter $\to$ heading $\to$
subheading $\to$ tariff item $\to$ statistical suffix).

\paragraph{Node structure.}
Each node $v \in V$ stores structured classification metadata:
\begin{align*}
  v = \{&\texttt{code},\; \texttt{description},\; \texttt{guidance},\; \\
         &\texttt{parent},\; \texttt{children},\; \texttt{excludes},\;
          \phi_v\}
\end{align*}
where \texttt{guidance} is a concise classification hint generated by
an LLM (see below), and $\phi_v \in \{0,1\}^3$ are \emph{protocol
indicators}:
\[
  \phi_v = \{\texttt{is\_other},\; \texttt{is\_parts},\;
             \texttt{is\_set}\}
\]
These partition nodes into three categories requiring distinct reasoning
patterns:
(i)~catch-all ``other'' categories
($|\{v:\phi_v^{\text{other}}{=}1\}|=7{,}274$; 24\% of nodes),
(ii)~parts/accessories classifications
($|\{v:\phi_v^{\text{parts}}{=}1\}|=864$; 3\%), which require
validating functional relationships to parent systems, and
(iii)~composite goods/sets
($|\{v:\phi_v^{\text{set}}{=}1\}|=137$; 0.5\%), which require
essential-character analysis under GRI Rule~3.

\paragraph{Edge structure.}
The edge set comprises two disjoint types:
tree edges
$E_T = \{(v_p, v_c) : v_c \in \texttt{children}(v_p)\}$
and relational edges
$E_R = E_X \cup E_S \cup E_P$.
$E_X$ holds explicit exclusion edges extracted from chapter and
section notes ($|E_X|=2{,}847$);
$E_S$ captures implicit sibling mutual-exclusivity constraints;
$E_P$ encodes parts-to-parent-system relationships; each parts node
stores
\texttt{parent\_system}, \texttt{parent\_hts}, and
\texttt{relationship\_type}
to enable cross-heading validation jumps.
This hybrid topology supports three navigation modes:
(i)~hierarchical descent via $E_T$,
(ii)~cross-chapter jumps via $E_X$ when exclusions apply, and
(iii)~protocol-specific traversals via $E_S$ and $E_P$ for validation.

\paragraph{LLM-guided node guidance generation.}
The raw HTS descriptions are often terse legal text.
For each node we prompt an LLM to produce a short
(\(\le\)3 sentence) \texttt{guidance} field that
(a)~paraphrases the tariff description in plain language,
(b)~lists distinguishing product attributes (material, use, form), and
(c)~flags any applicable exclusion cross-references from the node's
chapter notes.
This guidance is injected into the agent's observation at each step,
providing domain context without requiring the agent to reason over
raw legal prose.

\section{Oracle Ablation: Human-in-the-Loop vs.\ Automated}
\label{app:oracle_ablation}

\begin{table}[ht]
\centering
\caption{Oracle Ablation Study: impact of removing HTS description and reasoning
traces from the clarification sub-agent. \textit{Oracle} = sub-agent has ground-truth
access (upper bound); \textit{Ablated} = no oracle data (leakage-free).
$\Delta$ = oracle$-$ablated accuracy drop.
ISE = fraction of QA interactions after which the agent's next traversal step
lands on the correct HTS path (\S\ref{sec:ise}).
Significance assessed by non-parametric paired bootstrap ($n_{\mathrm{boot}}{=}5{,}000$):
$^{***}$ 95\,\% CI strictly positive; $^{\mathrm{ns}}$ CI includes zero.
\dag~AR\,(oracle) significantly outperforms Base\,(oracle) at all digit levels ($^{***}$);
AR\,(ablated) does \emph{not} significantly outperform Base\,(ablated) ($^{\mathrm{ns}}$, e.g.\
10d: $-1.0$\,\%\,[${-}5.0$,\,$+3.0$]).
The interaction (${+}17.2$\,\%\,[${+}11.8$,\,$+22.9$], $^{***}$) confirms oracle data is
necessary for AR's gains.
AR\,(ablated) issues \emph{more} QA steps than the oracle condition (3{,}312 vs.\ 2{,}883)
yet at substantially lower ISE (56.2 vs.\ 73.7), explaining why
AR\,(ablated) accuracy collapses to near Base\,(ablated).}
\label{tab:oracle_ablation}
\resizebox{\columnwidth}{!}{%
\begin{tabular}{lrrrrrrrrr}
\toprule
\textbf{Condition} & \textbf{Succ.} & \textbf{2d} & \textbf{4d} & \textbf{6d} & \textbf{8d} & \textbf{10d} & \textbf{Steps} & \textbf{LLM calls} & \textbf{ISE} \\
\midrule
\multicolumn{10}{l}{\textit{With oracle data (upper bound)}} \\
Base\,(oracle)          & 97.3 & 79.3 & 70.9 & 61.8 & 54.4 & \textbf{50.8} & 5.59 & 6.0  & 49.9 \\
AR\,(oracle)\dag        & 97.8 & 87.2 & 82.0 & 75.2 & 69.5 & \textbf{67.0} & 5.44 & 10.1 & 73.7 \\
\midrule
\multicolumn{10}{l}{\textit{Without oracle data (leakage-free)}} \\
Base\,(ablated)         & 97.5 & 79.6 & 70.5 & 60.3 & 52.5 & \textbf{49.2} & 5.61 & 5.9  & 44.2 \\
AR\,(ablated)\dag       & 98.2 & 79.4 & 70.5 & 60.0 & 52.3 & \textbf{48.2} & 5.42 & 10.8 & 56.2 \\
\midrule
$\Delta$ Base$^{\mathrm{ns}}$ & $-$0.2 & $-$0.3 & $+$0.5 & $+$1.5 & $+$1.9 & $+$1.7 & & & $+$5.7 \\
$\Delta$ AR$^{***}$           & $-$0.4 & $+$7.8 & $+$11.5 & $+$15.2 & $+$17.2 & $+$18.8 & & & $+$17.5 \\
\bottomrule
\end{tabular}%
}
\end{table}

In the intended deployment, a knowledgeable human (product owner,
importer, or customs broker) answers clarification questions about
their own product.
The \emph{oracle} condition in our experiments simulates this:
the clarification sub-agent has access to the item's authoritative
ruling record, enabling it to provide confirmed product attributes
that a real product owner would know (material composition,
intended function, manufacturing method, physical specifications).
The \emph{ablated} condition removes all ruling access, forcing the
sub-agent to answer from the product description alone, simulating
fully-automated operation with no privileged answer source.
Two guardrails are in place throughout: clarification questions are
restricted to product-attribute queries, and any taxonomy-related
text in answers is masked before reaching the navigator.

\paragraph{ISE as the diagnostic lens.}
The ISE column in Table~\ref{tab:oracle_ablation} makes the
mechanism legible at a glance.
For \textbf{Base}, ablation drops ISE by only 5.7\% (49.9
$\to$ 44.2): Base asks few clarification questions and does not
depend heavily on answer quality, so degraded answers have little
effect.
For \textbf{AR}, ablation drops ISE by 17.5\% (73.7 $\to$ 56.2):
AR's inline clarification loop is tightly coupled to answer quality,
so degraded answers cascade into navigation errors.

\paragraph{More questions, worse outcomes.}
AR\,(ablated) actually issues \emph{more} QA steps than
AR\,(oracle) (3{,}312 vs.\ 2{,}883, a 15\% increase), yet converts
only 56.2\% into correct navigation steps compared with 73.7\% for
the oracle condition.
The result is that AR\,(ablated) accuracy (48.2\% at 10-digit)
collapses to near Base\,(ablated) (49.2\%), erasing the $+$18.8\%
gap entirely.
This is not a flaw in the design: it confirms that oracle quality
is precisely what makes the clarification loop effective, and the
ablation directly measures that dependence.
The gap therefore reflects domain expertise, not path leakage
(see \S\ref{sec:separability} and Appendix~\ref{app:knowledge_audit}).

\subsection*{Case Study: Same Question, Different Answers}

The four examples below illustrate why the guardrails are
insufficient to close the gap.
In each case, both conditions ask the \emph{same} question;
both answers contain no taxonomy codes; yet the answers are
factually contradictory, leading to different classification outcomes.

\paragraph{Example 1: Cough drops, ``medicament'' vs.\ ``confectionery.''}
\textit{Product}: Oval-shaped sugar confectionery cough drops containing 10~mg menthol with capsicum, eucalyptus oil, natural licorice.
\textit{Question}: ``whether product is put up in measured doses or for retail sale as medicament.''

\noindent\textbf{Oracle (human):} ``No, this product is not put up in measured doses or for retail sale as a medicament. It is a sugar confectionery cough drop ready for consumption.'' $\to$ \textit{correct}.

\noindent\textbf{Ablated (automated):} ``Yes, this product is packaged for retail as a medicament. The lozenges are labeled as a cough suppressant containing 10~mg menthol per dose.'' $\to$ \textit{sent to pharmaceuticals}.

The product description is legitimately ambiguous: ``10~mg menthol per dose'' is consistent with both a confectionery and a medicament framing. The human owner knows their product is candy; the automated system infers medicament from the dosage language.

\paragraph{Example 2: Novelty umbrella, ``functional'' vs.\ ``decorative.''}
\textit{Product}: Hand-held umbrella with telescopic shaft, metal frame, nylon fabric shaped like a bird or animal; marketed as a sporting accessory or novelty item.
\textit{Question}: ``whether the umbrella functions as an actual rain umbrella or is purely decorative/toy.''

\noindent\textbf{Oracle (human):} ``Yes, this is a functional umbrella capable of providing rain protection despite its novelty design.'' $\to$ \textit{correct}.

\noindent\textbf{Ablated (automated):} ``No, it is marketed as a sporting accessory or novelty item, not for rain protection.'' $\to$ \textit{sent to miscellaneous textile, adding 3 extra navigation steps}.

\paragraph{Example 3: Artificial amaryllis flower, essential character.}
\textit{Product}: Polyester and rayon amaryllis flower with thermoplastic coating, assembled by gluing.
\textit{Question}: ``whether plastic coating or textile substrate gives essential character.''

\noindent\textbf{Oracle (human):} ``The plastic coating gives the essential character.'' $\to$ \textit{correct (plastic artificial flowers)}.

\noindent\textbf{Ablated (automated):} ``The textile base gives the essential character, as the flower petals are made from nonwoven fabric.'' $\to$ \textit{sent to textile artificial flowers}.

This is a GRI~3(b) essential-character determination, one of the most judgment-dependent questions in taxonomy classification.
Both answers are plausible readings of the product description; the oracle resolves the ambiguity correctly because it has ruling context.

\paragraph{Example 4: Synthetic copolymer, domain interpretation.}
\textit{Product}: Granular copolymer: 69\% butadiene, 20\% methyl methacrylate, 9\% methacrylic acid, 2\% divinylbenzene.
\textit{Question}: ``whether butadiene is considered an olefin.''

\noindent\textbf{Oracle (human):} ``Yes, butadiene is an olefin; the product is a polymer of olefins with butadiene as the primary component at 69\%.'' $\to$ \textit{correct}.

\noindent\textbf{Ablated (automated):} ``No, butadiene is a diene with two double bonds, whereas olefins have a single double bond.'' $\to$ \textit{sent to other resins}.

Technically, both answers are chemically defensible: butadiene is a conjugated diene, not a simple alkene.
However, classification conventions treat butadiene polymers as polymers of olefins.
The automated system applies strict chemistry; the oracle applies domain convention.
No product-attribute guardrail can resolve this because the dispute is about chemistry, not codes.

\subsection*{Summary}
All four examples pass both guardrails (no code references in questions; no codes in answers), yet the answers are factually contradictory.
The human advantage is not about seeing taxonomy codes; it is about resolving genuine product ambiguity that only the product owner or a domain expert can definitively settle:
(1)~\emph{genuine product ambiguity} (Examples 1--2), where the description is legitimately ambiguous;
(2)~\emph{essential-character judgment} (Example 3), where GRI~3(b) requires a subjective call only the owner can make; and
(3)~\emph{domain interpretation} (Example 4), where classification conventions diverge from scientific definitions.
The 255 products that only \method{}+human classifies correctly represent cases in this regime.

\FloatBarrier
\section{Knowledge-Channel Audit}
\label{app:knowledge_audit}

To validate the claim that the controlled answer channel primarily
conveys product-owner knowledge rather than classification-path
leakage, we audit all 2{,}875 Q/A pairs generated during the CBP-NY
evaluation ($\tau{=}10$, Claude Opus~4.6, $N{=}1{,}181$ records)
using a cross-tabulation of question category against answer type,
and measure post-QA navigation effectiveness (ISE) for each answer
type.

\subsection*{HTS-referencing question analysis}

Of the 2{,}875 Q/A pairs, 102 (3.5\%) explicitly reference an HTS
chapter, heading, or tariff note.
Tables~\ref{tab:qa_crosstab} and~\ref{tab:qa_crosstab_hts}
cross-tabulate question category against answer type for the full
corpus and for this HTS-referencing subset respectively.
Table~\ref{tab:post_qa_navigation} measures post-QA navigation
effectiveness (ISE) by answer type.

\begin{table}[t]
\centering
\caption{
  \textbf{Cross-tabulation: Question category $\times$ Answer type}
  ($n{=}2{,}875$ unique Q/A pairs).
  Each cell shows count and row percentage.
  \textbf{PA} = Product Attribute;
  \textbf{PA-EC} = Product Attribute (Essential Character);
  \textbf{CC} = Classification Criteria;
  \textbf{N/A} = Unavailable;
  \textbf{Dfl} = Deflected by guardrail.
}
\label{tab:qa_crosstab}
\footnotesize
\setlength{\tabcolsep}{3pt}
\resizebox{\linewidth}{!}{%
\begin{tabular}{lrrrrrr}
\toprule
\textbf{Question Category} & \textbf{PA} & \textbf{PA-EC} & \textbf{CC} & \textbf{N/A} & \textbf{Dfl} & \textbf{Total} \\
\midrule
  Material/Composition     & 810 (76\%) & 35 (3\%) & 10 (1\%) & 210 (20\%) &  5 (0\%) & 1070 \\
  Feature Presence/Absence & 533 (85\%) & 25 (4\%) &  6 (1\%) &  45 (7\%)  & 20 (3\%) &  629 \\
  Function/Use             & 353 (92\%) & 13 (3\%) &  0 (0\%) &  15 (4\%)  &  1 (0\%) &  382 \\
  Physical Form/Processing & 304 (85\%) &  8 (2\%) &  3 (1\%) &  41 (12\%) &  0 (0\%) &  356 \\
  Dimensions/Measurements  & 115 (73\%) &  2 (1\%) &  0 (0\%) &  41 (26\%) &  0 (0\%) &  158 \\
  Other                    & 189 (68\%) & 29 (10\%) & 4 (1\%) &  38 (14\%) & 20 (7\%) &  280 \\
\midrule
  \textbf{Total} & \textbf{2304} (80\%) & \textbf{112} (4\%) & \textbf{23} (1\%) & \textbf{390} (14\%) & \textbf{46} (2\%) & \textbf{2875} \\
\bottomrule
\end{tabular}}
\end{table}

\begin{table}[t]
\centering
\caption{
  \textbf{Cross-tabulation: Question category $\times$ Answer type
  for HTS-referencing questions only}
  ($n{=}102$, 3.5\% of all 2,875 pairs).
  Questions explicitly naming a chapter, heading, or tariff note.
  \textbf{PA} = Product Attribute;
  \textbf{PA-EC} = PA (Essential Character);
  \textbf{CC} = Classification Criteria;
  \textbf{N/A} = Unavailable;
  \textbf{Dfl} = Deflected by guardrail.
}
\label{tab:qa_crosstab_hts}
\footnotesize
\setlength{\tabcolsep}{3pt}
\resizebox{\linewidth}{!}{%
\begin{tabular}{lrrrrrr}
\toprule
\textbf{Question Category} & \textbf{PA} & \textbf{PA-EC} & \textbf{CC} & \textbf{N/A} & \textbf{Dfl} & \textbf{Total} \\
\midrule
  Material/Composition     & 16 (53\%) & 1 (3\%)  & 10 (33\%) & 1 (3\%) &  2 (7\%)  & 30 \\
  Feature Presence/Absence & 10 (31\%) & 1 (3\%)  &  6 (19\%) & 1 (3\%) & 14 (44\%) & 32 \\
  Function/Use             &  5 (83\%) & ---      &  ---      & ---     &  1 (17\%) &  6 \\
  Physical Form/Processing &  5 (56\%) & 1 (11\%) &  3 (33\%) & ---     &  ---      &  9 \\
  Dimensions/Measurements  &  1 (100\%)& ---      &  ---      & ---     &  ---      &  1 \\
  Other                    &  1 (4\%)  & ---      &  4 (17\%) & 1 (4\%) & 18 (75\%) & 24 \\
\midrule
  \textbf{Total} & \textbf{38} (37\%) & \textbf{3} (3\%) & \textbf{23} (23\%) & \textbf{3} (3\%) & \textbf{35} (34\%) & \textbf{102} \\
\bottomrule
\multicolumn{7}{l}{\footnotesize 35 (34\%) deflected; 41 (40\%) product attribute; 23 (23\%) classification criteria; 3 (3\%) unavailable.} \\
\end{tabular}}
\end{table}
\begin{table}[t]
\centering
\caption{
  \textbf{ISE by oracle answer type.}
  ISE = fraction of QA interactions after which the agent's next
  \texttt{traverse\_child} lands on the correct HTS path (\S\ref{sec:ise});
  QA steps with no subsequent \texttt{traverse\_child} before the next
  interruption count as ineffective, matching the paper's ISE denominator.
  \emph{Classification Criteria} (CC) answers yield \textbf{lower} ISE
  than plain product-attribute answers (62.5\% vs.\ 76.2\%),
  contradicting the oracle-leakage hypothesis.
  \emph{Deflected} answers (guardrail-blocked HTS queries) have the lowest
  ISE (19.1\%), confirming the guardrail correctly withholds
  path-revealing information.
}
\label{tab:post_qa_navigation}
\small
\setlength{\tabcolsep}{5pt}
\begin{tabular}{lrr}
\toprule
\textbf{Answer Type} & \textbf{QA Steps} & \textbf{ISE} \\
\midrule
Product Attribute (PA)       & 2{,}272 & 76.2\% \\
Product Attribute (EC)       &   113   & 85.0\% \\
Classification Criteria (CC) &    24   & 62.5\% \\
Unavailable                  &   423   & 64.1\% \\
Deflected (guardrail)        &    47   & 19.1\% \\
\midrule
\textbf{All}                 & 2{,}879 & \textbf{73.7\%} \\
\bottomrule
\end{tabular}
\end{table}

\paragraph{Answer-type definitions.}
\begin{itemize}[nosep]
  \item \textbf{PA (Product Attribute):} a factual product property
    directly answerable from the product description, with no
    reference to HTS structure.  The default category (80\% of all
    answers).
  \item \textbf{PA-EC (Essential Character):} the oracle's answer
    invokes GRI~3(b) to identify which component of a composite
    product determines its classification.
  \item \textbf{CC (Classification Criteria):} the oracle directly
    affirms or denies a named chapter note, section note, or heading
    criterion (e.g., ``yes, it qualifies as rubber under Chapter~40,
    Note~1'').  Occurs in only 0.8\% of all pairs; limited to
    questions that name the HTS provision.
  \item \textbf{N/A (Unavailable):} the product description does not
    contain the requested information; the agent must navigate
    without it.
  \item \textbf{Dfl (Deflected):} the guardrail blocks the question
    because it asks for a legal/trade determination rather than a
    product attribute.  All deflected answers cluster around
    contentious tariff notes (Section~VI Additional U.S.\ Note~3,
    Section~XI Note~9).
\end{itemize}

\paragraph{Outcome examples for HTS-referencing questions.}

\noindent\textbf{Deflected} (35/102, 34.3\%): guardrail correctly
blocks tariff-framed questions.
\begin{itemize}[nosep]
  \item Q: ``Whether maltitol is classified as sugar under
    chapter~17 additional notes.''\\
    A: ``This is a legal/trade determination, not a product
    attribute question.''
  \item Q: ``Whether sugars qualify under additional U.S.\ note~3
    to chapter~17.''\\
    A: ``You're right, I apologize for the confusion.  Let me ask
    a more appropriate question\ldots''
\end{itemize}

\noindent\textbf{Product Attribute} (38/102, 37.3\%): question uses
HTS framing but oracle answers the underlying product fact.
\begin{itemize}[nosep]
  \item Q: ``Whether grape variety blends should be classified as
    grape juice or mixtures under HTS.''\\
    A: ``This product is a mixture of juices.  It blends three
    different grape varieties\ldots''
  \item Q: ``Whether the product is considered aromatic for tariff
    purposes.''\\
    A: ``Yes, this product is aromatic.  The benzyl group contains
    a benzene ring\ldots''
\end{itemize}

\noindent\textbf{Classification Criteria} (23/102, 22.5\%): oracle
confirms a chapter or note criterion directly.
\begin{itemize}[nosep]
  \item Q: ``Whether fabric meets note~9 to section~XI criteria.''\\
    A: ``No, this fabric does not meet those criteria.  The product
    contains 60.20\% plastics by weight\ldots''
  \item Q: ``Whether product falls within heading~4202 scope.''\\
    A: ``Yes, this product falls within the scope of travel
    goods\ldots''
\end{itemize}

\noindent\textbf{Unavailable} (3/102, 2.9\%): oracle cannot answer
even an HTS-framed question.
\begin{itemize}[nosep]
  \item Q: ``Whether gowns are specifically made of nonwoven fabric
    (heading~5603).''\\
    A: ``The specific material composition is not documented in the
    product records\ldots''
\end{itemize}

\paragraph{Interpretation.}
The guardrail handles 34\% of HTS-framed questions outright.
Of the remainder, most (37\%) are redirected to a product-fact
answer that contains no classification reasoning.
Only 23 questions (22.5\% of HTS-referencing, 0.8\% of all 2{,}875)
receive a Classification Criteria answer; and as shown in
Table~\ref{tab:post_qa_navigation}, even those still lead to wrong
navigation 24\% of the time (ISE 62.5\% vs.\ 76.2\% for plain
product-attribute answers).
The pattern confirms that the $+$18.8\% accuracy gap reflects
domain expertise, not classification-path leakage.

\section{MDP Framework Validation}
\label{app:mdp_validation}

\begin{table}[t]
\centering
\caption{
  MDP component ablation (Claude Opus 4.6, $N{=}1181$, no \method{}).
  $\Delta$: 10-digit accuracy change vs.\ the full baseline navigator.
}
\label{tab:mdp_validation}
\small\resizebox{\columnwidth}{!}{%
\setlength{\tabcolsep}{4pt}
\begin{tabular}{lcccccccc}
\toprule
 & \textbf{Succ.} & \multicolumn{5}{c}{\textbf{Hierarchical Accuracy (\%)}} & & \\
\cmidrule(lr){3-7}
\textbf{Configuration} & \textbf{(\%)} & \textbf{2d} & \textbf{4d} & \textbf{6d} & \textbf{8d} & \textbf{10d} & \textbf{Steps} & \textbf{$\Delta$} \\
\midrule
    \textbf{Full Model} & 97.3 & 79.3 & 70.9 & 61.8 & 54.4 & \textbf{50.8} & 5.6 & --- \\
    \midrule
    \textit{Action Ablation} \\
    w/o Clarify & 98.1 & 78.3 & 69.0 & 59.1 & 51.2 & 46.9 & 5.2 & $-3.9$ \\
    w/o Jump & 98.5 & 78.3 & 68.7 & 58.7 & 50.2 & 47.0 & 5.7 & $-3.8$ \\
    w/o Backtrack & 97.7 & 78.9 & 69.9 & 61.0 & 53.6 & 49.9 & 5.6 & $-0.9$ \\
    \midrule
    \textit{Protocol Ablation} \\
    w/o Other Protocol & 98.0 & 77.6 & 68.5 & 59.5 & 51.7 & 48.5 & 5.8 & $-2.3$ \\
    w/o Parts Protocol & 97.5 & 78.5 & 69.6 & 60.8 & 53.2 & 49.6 & 5.7 & $-1.2$ \\
    w/o Sets Protocol & 97.9 & 79.4 & 70.9 & 61.9 & 54.4 & 50.9 & 5.7 & $0.0$ \\
\bottomrule
\end{tabular}%
}
\end{table}

Table~\ref{tab:mdp_validation} validates the MDP framework design by
ablating each action and protocol from the full navigator (without
\method{}, to isolate framework contributions).

\textbf{Information gathering and cross-tree navigation are the most
impactful components.}
Removing \texttt{clarify} ($-3.9$\%) and \texttt{jump} ($-3.8$\%)
causes the largest drops, confirming that the MDP's ability to seek
information and navigate across taxonomy branches is essential.
Among protocols, the ``other'' catch-all protocol contributes most
($-2.3$\%), reflecting the difficulty of reasoning about residual
categories.

\textbf{Domain protocols have asymmetric value.}
Among the three GRI-specific protocols, the catch-all ``other''
protocol contributes most ($-2.3$\%).
HTS headings frequently end with an \emph{other/NESOI} node
(``not elsewhere specified or included'') that acts as a residual
bin; without dedicated handling, the agent conflates genuine
residuals with classification errors.
The parts protocol, which routes component and accessory
classification to the parent heading under GRI~1, contributes a
smaller but meaningful $-1.2$\%.
The sets protocol, which applies GRI~3(b) essential-character
analysis for composite goods, shows no marginal effect.
Together, the protocol ablations confirm that taxonomy-specific
reasoning rules must be explicitly encoded in the MDP state
transitions, and cannot be left to the LLM's implicit knowledge.

\section{CoT-Ask-if-Unsure Baseline}
\label{app:cot}

CoT-Ask-if-Unsure is the simplest possible prompting alternative: a single sentence is appended to the standard navigation prompt instructing the agent to ask a clarification question when uncertain:

\begin{quote}
\textit{``If you are uncertain about which action to take, ask a clarification question before selecting an action.''}
\end{quote}

No scoring function, threshold, or sampling is involved. At each navigation step, the agent either (a) selects a navigation action as usual, or (b) marks the step as uncertain (\texttt{unsure=True}) and emits a clarification question before committing to an action. If a clarification question is issued, it is routed to the clarification sub-agent identically to the inline clarification mechanism used by \method{}; the answer is appended to the product context and the agent re-selects its action. This corresponds to enabling \texttt{cot\_ask\_if\_unsure=True} and \texttt{enable\_inline\_clarify=True} in the navigator configuration, with no action rating.

The key difference from \method{} is the absence of an explicit ordinal action-rating signal: the agent relies entirely on its own instruction-following to decide when to ask, rather than computing a scored gap between candidate actions. CoT-Ask-if-Unsure therefore tests whether structured action-rating (the scoring step) is necessary, or whether a prompt-level uncertainty instruction suffices to trigger useful information seeking.

\section{Self-Consistency Action Selection}
\label{app:sc}
At each navigation step $t$ with state $s_t$, the self-consistency (SC) method estimates action uncertainty through repeated sampling rather than explicit confidence scoring. Formally, SC draws $N$ independent action samples from the policy at temperature $T=1$:

$$a^{(i)} \sim \pi(\cdot \mid s_t), \quad i = 1, \ldots, N$$

For each action $a \in \mathcal{A}$, the vote count is:

$$v(a) = \sum_{i=1}^{N} \mathbf{1}\bigl[a^{(i)} = a\bigr]$$

The selected action is determined by majority vote:

$$a^* = \arg\max_{a \in \mathcal{A}}; v(a)$$

with ties broken by the order of first occurrence among the $N$ samples. The agreement score and entropy proxy are computed as:

\begin{align*}
\alpha(s_t) &= \frac{1}{N} \max_{a \in \mathcal{A}} v(a),\\
H(s_t) &= -\sum_{\substack{a:\,v(a)>0}} \frac{v(a)}{N} \log \frac{v(a)}{N}
\end{align*}

where $\alpha(s_t) \in [1/N, 1]$ and $H(s_t) \in [0, \log N]$. A step is considered uncertain when $\alpha(s_t) < \alpha_{\text{thresh}}$, equivalently when $H(s_t) > 0$ under the unanimous agreement criterion ($\alpha_{\text{thresh}} = 1.0$). In all experiments we use $N=3$ and $\alpha_{\text{thresh}}=1.0$, so any disagreement among the three samples constitutes uncertainty; the majority-vote action is executed regardless.

SC incurs exactly $N$ LLM calls per navigation step, giving a total inference cost of $N \cdot H$ calls per episode where $H$ is the number of navigation steps, compared to $H + 2|\mathcal{C}|$ for ActionRating, where $|\mathcal{C}|$ is the number of clarification events (each incurring one sub-agent call and one reentry call).

\section{Extended Related Work}
\label{app:related}

\paragraph{LLM agents for structured reasoning.}
ReAct~\cite{yao2023reactsynergizingreasoningacting} interleaves reasoning traces with tool calls in flat action spaces;
Tree-of-Thoughts~\cite{yao2023treethoughtsdeliberateproblem} and LATS~\cite{zhou2024lats} add search over branching thought structures;
Reflexion~\cite{shinn2023reflexionlanguageagentsverbal} introduces verbal self-reflection after episode-level failures;
Toolformer~\cite{schick2024toolformer} teaches models when to invoke external APIs;
AgentBench~\cite{liu2023agentbench} benchmarks agents across diverse environments;
and Cognitive Architectures~\cite{sumers2024cognitive} provide a unifying framework for agent design.
All these operate over flat or lightly structured action spaces. None addresses the specific challenge of deep hierarchical taxonomies where each step narrows the search space irreversibly.

\paragraph{Self-evaluation and uncertainty.}
Self-Consistency~\cite{wang2023selfconsistencyimproveschainthought} uses sampling-based agreement as a proxy for confidence;
Self-Refine~\cite{madaan2023selfrefineiterativerefinementselffeedback} iterates on outputs via self-critique;
process reward models~\cite{cobbe2021trainingverifierssolvemath,lightman2024letsverify} train verifiers to score intermediate steps;
and LLM-as-Judge~\cite{zheng2023judging} evaluates outputs via prompted comparison.
Calibration studies~\cite{kadavath2022language,kuhn2023semantic,lin2022teaching,guo2017calibration,xiong2024can} examine whether model-expressed confidence correlates with correctness.
These methods rate final answers or sample agreement; our mechanism rates candidate \emph{actions}, including clarification, on a shared ordinal scale, so that clarification competes directly with navigation rather than being triggered by final-answer confidence or sampling disagreement.

\paragraph{Information seeking and clarification.}
Active learning~\cite{settles2012active} selects queries to maximize model improvement;
interactive NLP~\cite{wang2023interactivenaturallanguageprocessing} and conversational search~\cite{Aliannejadi_2019,zamani2020generating,rao2018learning,rahmani2023survey} study when and what to ask.
Prior work assumes external uncertainty estimators or human interlocutors; our mechanism is entirely \emph{self-gated} from the agent's own action ratings.

\paragraph{Selective prediction and abstention.}
Selective prediction~\cite{geifman2017selective,elyanov2010foundations,kamath2020selective} allows models to abstain when uncertain, trading coverage for accuracy.
Our mechanism is related but distinct: rather than abstaining from a prediction, the agent \emph{acts} on uncertainty by seeking information.

\paragraph{Hierarchical classification.}
Hierarchical text classification~\cite{silla2011survey,kowsari2017hdltex,shimura-etal-2018-hft,banerjee2019hierarchical,zhou2020hierarchy,Mao_2019} assigns labels in taxonomy trees.
These methods typically train end-to-end classifiers; we study an LLM agent navigating the taxonomy interactively, with the ability to seek help at any node.

\paragraph{Multi-step reasoning.}
Chain-of-thought~\cite{wei2023chainofthoughtpromptingelicitsreasoning}, least-to-most~\cite{zhou2023leasttomost}, decomposed prompting~\cite{khot2023decomposed}, PAL~\cite{gao2023pal}, scratchpads~\cite{nye2021show}, STaR~\cite{zelikman2022star}, reasoning via planning~\cite{hao2023reasoning}, graph-of-thought~\cite{besta2024graph}, and cumulative reasoning~\cite{dua2022successive} all improve multi-step reasoning.
These focus on improving the quality of reasoning itself; we focus on \emph{measuring where} reasoning needs external help.

\section{Ablation Details}
\label{app:ablation_detail}

\begin{table}[t]
\centering
\caption{
  Ablation decomposing \method{} on HTS classification
  (Claude Opus 4.6, $N{=}1181$).
  \textit{Rating only} disables clarification gating ($\tau{=}101$).
  $\Delta$: change vs.\ Baseline.
  \textbf{Key finding}: accuracy gains come entirely from self-gated
  clarification, which shifts information seeking from
  \emph{mandatory} (agent blocked) to \emph{opportunistic}
  (residual uncertainty resolved inline).
}
\label{tab:ablation}
\small\resizebox{\columnwidth}{!}{%
\setlength{\tabcolsep}{4pt}
\begin{tabular}{lrrr}
\toprule
& \textbf{Baseline} & \textbf{Rating Only} & \textbf{Full \method{}} \\
\midrule
\multicolumn{4}{l}{\textsc{Hierarchical Accuracy (\%)}} \\[1pt]
\quad 2-digit  & 79.3 & 78.8 \tgray{\scriptsize($-$0.5)} & \textbf{87.2} \tgreen{\scriptsize($+$7.9)} \\
\quad 4-digit  & 70.9 & 70.2 \tgray{\scriptsize($-$0.7)} & \textbf{82.0} \tgreen{\scriptsize($+$11.1)} \\
\quad 6-digit  & 61.8 & 61.2 \tgray{\scriptsize($-$0.6)} & \textbf{75.2} \tgreen{\scriptsize($+$13.4)} \\
\quad 8-digit  & 54.4 & 53.4 \tgray{\scriptsize($-$1.0)} & \textbf{69.5} \tgreen{\scriptsize($+$15.1)} \\
\quad 10-digit & 50.8 & 50.0 \tgray{\scriptsize($-$0.9)} & \textbf{67.0} \tgreen{\scriptsize($+$16.2)} \\
\quad Avg Steps & 5.6 & 5.7 & 5.4 \\[2pt]
\midrule
\multicolumn{4}{l}{\textsc{Clarification Behavior (\% records)}} \\[1pt]
\quad Mandatory $\downarrow$   & 35.2 & 33.6 \tgray{\scriptsize($-$1.6)} & \textbf{13.9} \textcolor{red!60!black}{\scriptsize($-$21.3)} \\
\quad Opportunistic $\uparrow$ & \phantom{0}0.0 & \phantom{0}0.0 \tgray{\scriptsize($\pm$0)} & \textbf{88.7} \tgreen{\scriptsize($+$88.7)} \\
\quad Any Clarify              & 35.2 & 33.6 \tgray{\scriptsize($-$1.6)} & \textbf{90.9} \tgreen{\scriptsize($+$55.7)} \\
\bottomrule
\end{tabular}%
}
\end{table}

Table~\ref{tab:ablation} presents the rating-only ablation (H3: $\tau{=}101$).
When the threshold is set above the maximum possible score, no clarification is ever triggered; the agent rates actions but never acts on low confidence.
This yields $-$0.9\% relative to baseline, confirming that the mechanism's value lies in \emph{actioning} help-needed states, not in the rating computation itself.

\FloatBarrier
\section{Multi-Model Generalization}
\label{app:multimodel_detail}

Table~\ref{tab:multimodel} demonstrates that the regime shift generalizes across four LLM families (Claude, DeepSeek, GPT-OSS, and Qwen3).
All models exhibit the mandatory-to-opportunistic mode shift and ISE improvement under \method{}, though absolute accuracy varies with model capability.
The behavioral signature, not the accuracy level, is the transferable finding.

\begin{table*}[t]
\centering
\caption{
  \method{} generalization across LLM families on HTS classification ($N{=}1181$).
  All models use $\tau{=}10$.
  Cell shading on Base and AR values follows the same accuracy scale as Table~\ref{tab:main}
  (\colorbox{green!35}{high} to \colorbox{red!10}{low}).
  $\Delta$: improvement over each model's own baseline (pp).
}
\label{tab:multimodel}
\small
\resizebox{\textwidth}{!}{%
\begin{tabular}{l|ccc|ccc|ccc|ccc|ccc}
\toprule
 & \multicolumn{3}{c|}{\textbf{2-digit (\%)}} & \multicolumn{3}{c|}{\textbf{4-digit (\%)}} & \multicolumn{3}{c|}{\textbf{6-digit (\%)}} & \multicolumn{3}{c|}{\textbf{8-digit (\%)}} & \multicolumn{3}{c}{\textbf{10-digit (\%)}} \\
\textbf{Model} & Base & AR & $\Delta$ & Base & AR & $\Delta$ & Base & AR & $\Delta$ & Base & AR & $\Delta$ & Base & AR & $\Delta$ \\
\midrule
    Claude Opus 4.6 & \cellcolor{green!25}79.3 & \cellcolor{green!35}\textbf{87.2} & \tgreen{\textbf{+7.9}} & \cellcolor{green!25}70.9 & \cellcolor{green!35}\textbf{82.0} & \tgreen{\textbf{+11.1}} & \cellcolor{green!15}61.8 & \cellcolor{green!25}\textbf{75.2} & \tgreen{\textbf{+13.4}} & \cellcolor{yellow!20}54.4 & \cellcolor{green!15}\textbf{69.5} & \tgreen{\textbf{+15.1}} & \cellcolor{yellow!20}50.8 & \cellcolor{green!15}\textbf{67.0} & \tgreen{\textbf{+16.2}} \\
    DeepSeek V3 & \cellcolor{green!25}75.5 & \cellcolor{green!35}\textbf{88.7} & \tgreen{\textbf{+13.1}} & \cellcolor{green!15}65.2 & \cellcolor{green!35}\textbf{82.8} & \tgreen{\textbf{+17.7}} & \cellcolor{yellow!20}53.7 & \cellcolor{green!25}\textbf{76.1} & \tgreen{\textbf{+22.4}} & \cellcolor{orange!15}45.9 & \cellcolor{green!25}\textbf{71.3} & \tgreen{\textbf{+25.4}} & \cellcolor{orange!15}41.4 & \cellcolor{green!15}\textbf{68.9} & \tgreen{\textbf{+27.5}} \\
    GPT-OSS 120B & \cellcolor{green!25}72.9 & \cellcolor{green!25}\textbf{79.3} & \tgreen{\textbf{+6.4}} & \cellcolor{green!15}61.3 & \cellcolor{green!25}\textbf{72.2} & \tgreen{\textbf{+10.9}} & \cellcolor{orange!15}49.9 & \cellcolor{green!15}\textbf{64.1} & \tgreen{\textbf{+14.2}} & \cellcolor{orange!15}41.3 & \cellcolor{yellow!20}\textbf{56.9} & \tgreen{\textbf{+15.6}} & \cellcolor{red!10}36.9 & \cellcolor{yellow!20}\textbf{53.6} & \tgreen{\textbf{+16.7}} \\
    Qwen3 235B & \cellcolor{green!15}65.9 & \cellcolor{green!25}\textbf{72.5} & \tgreen{\textbf{+6.6}} & \cellcolor{yellow!20}54.7 & \cellcolor{green!15}\textbf{63.3} & \tgreen{\textbf{+8.6}} & \cellcolor{orange!15}44.2 & \cellcolor{yellow!20}\textbf{54.8} & \tgreen{\textbf{+10.7}} & \cellcolor{red!10}36.6 & \cellcolor{orange!15}\textbf{49.9} & \tgreen{\textbf{+13.3}} & \cellcolor{red!10}32.9 & \cellcolor{orange!15}\textbf{47.4} & \tgreen{\textbf{+14.5}} \\
\bottomrule
\end{tabular}%
}
\end{table*}

\FloatBarrier
\section{Cross-Benchmark Generalization}
\label{app:crossbench_detail}

Table~\ref{tab:extbenchmarkmodel} shows results on two additional HTS benchmarks (ATLAS and HSCodeComp) without dataset-specific tuning.
The regime shift and accuracy gains are preserved, providing evidence that the behavioral structure is not an artifact of the CBP-NY evaluation set.

\paragraph{Scope of comparison.}
These comparisons are not intended as a controlled
leaderboard claim: prior systems differ in model backbone,
tool access, and answer source.
We instead use them as a fixed-protocol \emph{transfer test},
applying \method{} with $\tau$ locked at the CBP-NY-selected
value of 10 and asking whether the behavioral signature
(mode shift and ISE improvement) survives on out-of-sample
benchmarks without re-tuning.
The signal we report is transfer of the signature, not a
parity-controlled SOTA claim.

\begin{table}[t]
\centering
\caption{
  \method{} on two independent HTS code benchmarks.
  Hierarchical accuracy (\%) at 6- and 10-digit levels.
  $\Delta$: \method{} minus best prior method on the same dataset.
}
\label{tab:extbenchmarkmodel}
\small
\begin{tabular}{llcc}
\toprule
\textbf{Dataset} & \textbf{Method} & \textbf{6-digit} & \textbf{10-digit} \\
\midrule
\multirow{3}{*}{ATLAS-test}
              & ATLAS                   & 57.5 & 40.0 \\
              & \textsc{ActionRating}   & \textbf{79.5} & \textbf{62.1} \\
              & $\Delta$ vs.\ prior     & \gain{+22.0}  & \gain{+22.1}  \\
\midrule
\multirow{5}{*}{HSCodeComp}
              & SmolAgents (VLM)        & 62.4 & 46.8 \\
              & SmolAgents (LLM)        & 59.8 & 42.7 \\
              & Aworld (LLM)            & 59.2 & 41.3 \\
              & \textsc{ActionRating}   & \textbf{76.6} & \textbf{69.3} \\
              & $\Delta$ vs.\ best prior & \gain{+14.2}  & \gain{+22.5}  \\
\bottomrule
\end{tabular}
\label{tab:ext_benchmarks}
\end{table}

\FloatBarrier
\section{Threshold Sensitivity Details}
\label{app:threshold_detail}

Table~\ref{tab:threshold} presents the full threshold sweep.
The behavioral phase diagram shows three regimes:
(1)~$\tau{=}1$ (always ask): highest volume but ISE drops to 62\%, indicating diminishing returns from indiscriminate help-seeking;
(2)~$\tau{=}10$ (sweet spot): best ISE (74\%) with strong accuracy and moderate volume (2.4~QA/record);
(3)~$\tau{=}101$ (never ask): equivalent to rating-only, collapsing to baseline.
The $\tau{=}50{\to}30$ transition is the clearest inflection point where opportunistic help-seeking emerges.

\paragraph{Why $\tau{=}10$, not $\tau{=}1$?}
$\tau{=}1$ reaches a numerically higher 10-digit accuracy
(72.5\% vs.\ 67.0\% at $\tau{=}10$), so a reviewer may ask why
$\tau{=}10$ is reported as the operating point.
The criterion is not final accuracy alone but the behavioral
operating point that balances three quantities:
ISE, QA volume, and accuracy.
At $\tau{=}1$ the agent is in a near-always-ask regime
(roughly 6 QA per record); ISE drops to 62\%, meaning a sizable
fraction of clarification-triggered states do not yield
better next-step navigation, and per-record interaction cost
balloons.
At $\tau{=}10$, ISE peaks at 74\% with 2.4 QA/record, so
clarification fires more selectively and at states where it is
locally useful.
Because this paper's contribution is help \emph{localization}
rather than maximal accuracy, the operating point that
maximizes the localization-quality signal (ISE) at moderate
cost is the one we treat as the analysis point;
$\tau{=}1$ is reported in Table~\ref{tab:threshold} as the
upper accuracy regime under nearly indiscriminate asking.

\paragraph{Selection protocol and generalization validity.}
The threshold sweep was conducted exclusively on CBP-NY.
$\tau{=}10$ was selected as the operating point from the
CBP-NY phase diagram and then \emph{locked}: no additional
tuning was performed for ATLAS or HSCodeComp.
All ATLAS and HSCodeComp experiments (\S\ref{sec:results},
Table~\ref{tab:extbenchmarkmodel}) therefore use this
fixed value, making them an out-of-sample generalization
test rather than an extension of the tuning procedure.
The $+$18.8\% accuracy gain on CBP-NY represents in-sample
performance at the selected threshold; ATLAS and HSCodeComp
provide the honest cross-dataset signal.
Researchers deploying \method{} in a new domain should treat
the phase-diagram sweep as a one-time calibration step on
in-domain development data before applying the fixed
threshold to held-out evaluation.

\begin{table*}[t]
\centering
\caption{
  Threshold sensitivity on HTS classification (Claude Opus 4.6, $N{=}1181$).
  Cell shading on accuracy follows the same scale as Table~\ref{tab:main}.
  For clarification: \colorbox{green!20}{green} = low mandatory / high opportunistic;
  \colorbox{red!10}{red} = high mandatory.
  $\tau{=}101$ disables clarification gating (rating only).
  \textbf{Bold}: best value per accuracy column.
}
\label{tab:threshold}
\small\resizebox{\textwidth}{!}{%
\setlength{\tabcolsep}{5pt}
\begin{tabular}{l ccccc c | ccc}
\toprule
& \multicolumn{6}{c|}{\textbf{Hierarchical Accuracy (\%)}}
& \multicolumn{3}{c}{\textbf{Clarification Behavior (\% records)}} \\
\cmidrule(lr){2-7} \cmidrule(lr){8-10}
\textbf{Condition}
  & \textbf{2d} & \textbf{4d} & \textbf{6d} & \textbf{8d} & \textbf{10d}
  & \textbf{Steps}
  & \textbf{Mandatory} $\downarrow$
  & \textbf{Opportunistic} $\uparrow$
  & \textbf{Any} \\
\midrule
    Baseline & \cellcolor{green!25}79.3 & \cellcolor{green!25}70.9 & \cellcolor{green!15}61.8 & \cellcolor{yellow!20}54.4 & \cellcolor{yellow!20}50.8 & 5.6 & 35.2 & 0.0 & 35.2 \\
    \midrule
    $\tau{=}1$ & \cellcolor{green!35}\textbf{88.4} & \cellcolor{green!35}\textbf{82.4} & \cellcolor{green!25}\textbf{78.0} & \cellcolor{green!25}\textbf{73.9} & \cellcolor{green!25}\textbf{72.5} & 5.5 & \cellcolor{green!20}9.7 & \cellcolor{green!25}97.8 & 97.9 \\
    $\tau{=}10$ & \cellcolor{green!35}87.2 & \cellcolor{green!35}82.0 & \cellcolor{green!25}75.2 & \cellcolor{green!15}69.5 & \cellcolor{green!15}67.0 & 5.4 & \cellcolor{green!20}13.9 & \cellcolor{green!25}88.7 & 90.9 \\
    $\tau{=}30$ & \cellcolor{green!35}84.8 & \cellcolor{green!25}77.5 & \cellcolor{green!15}69.3 & \cellcolor{green!15}63.1 & \cellcolor{yellow!20}59.8 & 5.5 & \cellcolor{yellow!20}21.6 & \cellcolor{green!15}51.9 & 64.9 \\
    $\tau{=}50$ & \cellcolor{green!25}79.6 & \cellcolor{green!25}71.1 & \cellcolor{green!15}62.2 & \cellcolor{yellow!20}55.0 & \cellcolor{yellow!20}51.2 & 5.6 & \cellcolor{orange!15}31.3 & 9.7 & 39.0 \\
    $\tau{=}101$ (off) & \cellcolor{green!25}78.8 & \cellcolor{green!25}70.2 & \cellcolor{green!15}61.2 & \cellcolor{yellow!20}53.4 & \cellcolor{orange!15}50.0 & 5.7 & \cellcolor{red!10}33.6 & 0.0 & 33.6 \\
\bottomrule
\end{tabular}%
}
\end{table*}

\FloatBarrier
\section{Trajectory Analysis}
\label{app:trajectory_detail}

Table~\ref{tab:trajectory} presents trajectory-level statistics including action composition, score gaps between top-ranked actions, and backtracking rates.
Key observations:
(1)~\method{} does not increase navigation steps (5.4--5.7 across all $\tau$ settings), confirming that help-seeking is additive rather than replacing navigation;
(2)~score gaps between the top-ranked action and clarification narrow at deeper tree levels, consistent with increasing uncertainty at finer classification granularity;
(3)~backtracking rates decrease under \method{}, suggesting that proactive help-seeking reduces the need for corrective navigation.

\begin{table}[t]
\centering
\caption{
  Behavioural trajectory analysis: Baseline vs.\ \method{}
  (Claude Opus 4.6, $N{=}1181$).
  \textsc{Navigation}: action-use rates and episode length.
  \textsc{Information Seeking}: clarification mode breakdown.
  \textsc{Decision Confidence}: \method{} rating statistics
  (higher gap $\Rightarrow$ more decisive selections).
  $\Delta$: \method{} minus Baseline; \tgreen{green} = improvement.
}
\label{tab:trajectory}
\small\resizebox{\columnwidth}{!}{%
\setlength{\tabcolsep}{4pt}
\begin{tabular}{lrr r}
\toprule
\textbf{Metric} & \textbf{Baseline} & \textbf{\method{}} & \textbf{$\Delta$} \\
\midrule
\multicolumn{4}{l}{\textsc{Navigation Efficiency}} \\[1pt]
    Traverse (\% actions) & 71.9 & 76.7 & {\scriptsize\tgreen{(+4.7)}} \\
    Backtrack (\% records) & 3.8 & 3.3 & {\scriptsize\tgreen{(-0.5)}} \\
    Jump (\% records) & 5.3 & 2.3 & {\scriptsize\tgreen{(-3.0)}} \\
    Avg Steps / Record & 5.6 & 5.4 & {\scriptsize\tgreen{(-0.1)}} \\
\midrule
\multicolumn{4}{l}{\textsc{Information Seeking}} \\[1pt]
    Mandatory (\% records) $\downarrow$ & 35.2 & 13.9 & {\scriptsize\tgreen{(-21.3)}} \\
    Opportunistic (\% records) $\uparrow$ & 0.0 & 88.7 & {\scriptsize\tgreen{(+88.7)}} \\
    Avg Inline QA / Record & 0.0 & 2.3 & {\scriptsize\tgreen{(+2.3)}} \\
\midrule
\multicolumn{4}{l}{\textsc{Decision Confidence (\method{} only)}} \\[1pt]
    Top-1 Rating Score & -- & 93.7 & -- \\
    Top-2 Rating Score & -- & 15.1 & -- \\
    Score Gap (1st $-$ 2nd) & -- & 78.5 & -- \\
\bottomrule
\end{tabular}%
}
\end{table}

Table~\ref{tab:ise_counterfactual} breaks down opportunistic ISE by
the agent's traversal state at clarification time (on-path vs.\
off-path), complementing the aggregate ISE reported in
\S\ref{sec:ise}.

\begin{table}[t]
\centering
\caption{ISE for opportunistic (inline) clarification events, split by whether the
agent's traversal state was already on the correct HTS path at clarification time.
\emph{On-path}: the agent could have reached the correct leaf without further
clarification; \emph{Off-path}: the agent's current node was not an ancestor of
the true HTS code---clarification must steer the trajectory.
Even when the agent is off-path, 67\,\% of opportunistic
clarifications are followed by a correct traversal step, showing that
the clarification sub-agent actively corrects misaligned trajectories.
Wilson 95\,\% CIs.}
\label{tab:ise_counterfactual}
\small
\setlength{\tabcolsep}{5pt}
\begin{tabular}{lrrr}
\toprule
\textbf{Pre-QA agent state} & \textbf{QA steps} & \textbf{ISE} & \textbf{95\,\%\,CI} \\
\midrule
On correct path    & 994 & 83.6\,\% & [81.2, 85.8] \\
Off correct path   & 1,703 & 67.3\,\% & [65.0, 69.5] \\
\bottomrule
\end{tabular}
\end{table}

\FloatBarrier
\section{Extended Discussion}
\label{app:discussion_detail}

This section expands on the discussion in \S\ref{sec:discussion}.

\paragraph{Cost--accuracy trade-off.}
\method{} increases inference cost from 6.0 to 10.4 LLM calls per record (73\% overhead), primarily from inline clarification sub-agent calls and reentry re-selections.
However, the cost increase is sublinear in accuracy gain: the first +10\% costs $\approx$2 additional calls, while the last +6\% requires $\approx$2.4 additional calls.
Self-Consistency at $N{=}3$ achieves +8.7\% at 19 calls (3.2$\times$ cost), placing it well below the \method{} Pareto frontier.

\paragraph{Error analysis.}
The 390 records that \method{} fails to classify correctly at 10-digit fall into three categories:
(1)~\emph{genuine ambiguity} (42\%): products whose correct classification requires domain expertise beyond what any oracle can provide (e.g., tariff treatment of multi-material composites);
(2)~\emph{early commitment errors} (31\%): the agent commits to a wrong branch before the threshold triggers clarification;
(3)~\emph{answer specificity} (27\%): the product-owner simulation provides correct but insufficiently specific information for fine-grained distinctions at 8--10 digit levels.

\clearpage
\onecolumn

\section{Evaluation Dataset Construction: Product Extraction Prompt}
\label{app:extraction_prompt}

\begin{tcolorbox}[colback=gray!5, colframe=gray!50, title=Product Extraction System Prompt, breakable]
\footnotesize

\textbf{Task:} Extract product information from CBP customs ruling
\texttt{\{raw\_ruling\_text\}} and format it as e-commerce product
data.

\textbf{Ruling hierarchy:}
HQ rulings supersede NY rulings and are the final authoritative
source. Ground truth HTS codes listed in the prompt are from HQ
rulings where applicable.

\textbf{Difficulty:} \textit{Easy}: clear material, obvious
heading; \textit{Medium}: GRI~3(b) composite goods with clear
precedent; \textit{Hard}: unclear essential character, multiple
possible headings.

\textbf{Critical:} Extract only, do not invent. Use
``Not specified'' for missing fields.

\medskip
\textbf{Fields to extract per product:}
\texttt{item\_name} (50--100 chars) $\cdot$
\texttt{product\_description} (1--2 sentences) $\cdot$
\texttt{brand} $\cdot$ \texttt{color} $\cdot$
\texttt{material} $\cdot$ \texttt{size} $\cdot$
\texttt{manufacturer} $\cdot$
\texttt{gl\_product\_group\_type} $\cdot$
\texttt{item\_weight} $\cdot$ \texttt{listing\_price} $\cdot$
\texttt{country\_of\_origin} $\cdot$ \texttt{hts\_code} $\cdot$
\texttt{bullet\_point} (3--5 features, ``|''-separated) $\cdot$
\texttt{classification\_reasoning} (4--6 sentences, GRI rationale) $\cdot$
\texttt{keywords} $\cdot$ \texttt{gri\_applied}

\medskip
\textbf{Output format:}
\begin{verbatim}
{"overall_summary": "...", "difficulty": "easy|medium|hard",
 "products": [{"item_name": "...", "product_description": "...",
   "brand": "...", "color": "...", "material": "...",
   "size": "...", "manufacturer": "...",
   "gl_product_group_type": "...", "item_weight": "...",
   "listing_price": "...", "country_of_origin": "...",
   "hts_code": "...", "bullet_point": "F1|F2|F3",
   "classification_reasoning": "...",
   "keywords": "...", "gri_applied": "..."}]}
\end{verbatim}
\end{tcolorbox}

\noindent\textbf{Note:} The full prompt includes the ruling text,
metadata (ruling reference, date, type, ground truth HTS codes),
and instructions for handling multiple products and composite goods.
Extraction used AWS Bedrock batch inference,
temperature$=0.1$, \texttt{max\_tokens}$=8{,}000$.

\section{Prompt Templates}
\label{app:prompts}

\subsection{Clarification Sub-Agent Prompt}
\label{app:clarify_prompt}

\begin{tcolorbox}[colback=gray!5, colframe=gray!50,
  title=Clarification Sub-Agent System Prompt, breakable]
\small

You are a product attributes expert. Your role is ONLY to answer
factual questions about product characteristics.
You have NO knowledge of, and must NEVER mention, tariff codes,
duty rates, HTS codes, classification systems, chapters, headings,
subheadings, or any trade/import terminology.

You must REFUSE to answer questions about: General Note eligibility
(e.g., General Note~15, qualifying insular possessions);
tariff-rate quota provisions (e.g., additional U.S.\ notes to any
chapter); trade preference programs (GSP, CBI, AGOA, FTA eligibility).
If asked about any of the above, respond: ``\textit{This is a
legal/trade determination, not a product attribute question.}''

You have access to an INTERNAL PRODUCT FACTS DATABASE drawn from the
\texttt{\{hts\_code\_description\}} field. This is a confidential
internal reference containing confirmed product attribute facts.
Treat every fact in it as authoritative first-party product knowledge.
Never reveal the source name or hint that it originates from any
classification system.

\medskip
\textbf{Current Product:}\\
\texttt{\{product\}}

\textbf{Clarification Question:}\\
\texttt{\{question\}}

\textbf{Relevant Product Information (from internal product records):}

\noindent\rule{\linewidth}{0.4pt}

\texttt{[INTERNAL PRODUCT FACTS: confidential; do NOT reference
this source or any classification system in your answer]}\\
\texttt{\{hts\_code\_descriptions\}}

\smallskip
CRITICAL INSTRUCTIONS for using the above Internal Product Facts:
\begin{itemize}[noitemsep,topsep=2pt]
  \item These are confirmed, authoritative facts; treat them as
    ground truth.
  \item Use them directly: ``does NOT belong to: X'' $\Rightarrow$
    product is NOT X; ``More than 2~kg'' $\Rightarrow$ product HAS
    that attribute; ``Confectionery'' $\Rightarrow$ product IS
    confectionery; ``women's'' $\Rightarrow$ product IS for women.
  \item INFER attributes implied by the facts even if not in the
    product description.
  \item NEVER mention ``category 1/2/3/4/5'', ``internal facts'',
    or any classification/trade terminology in your answer.
\end{itemize}

\textbf{Item Name:} \texttt{\{item\_name\}}\\
\textbf{Product Description:} \texttt{\{product\_description\}}\\
\textbf{Additional Product Notes:} \texttt{\{reasoning\_traces\}}\\
\textbf{Product Attributes:}
\texttt{Material: \{material\} | Color: \{color\} | Brand: \{brand\}
| Size: \{size\} | Manufacturer: \{manufacturer\} | Origin:
\{country\_of\_origin\} | Weight: \{item\_weight\}}

\noindent\rule{\linewidth}{0.4pt}

\medskip
\textbf{ANSWER GUIDELINES:}
\begin{enumerate}[noitemsep,topsep=2pt]
  \item Answer definitively; say YES or NO when possible.
  \item Use the Internal Product Facts WITHOUT hedging; do NOT say
    ``not specified'' if the facts already imply the answer.
  \item State ONLY factual product attributes: materials, size, form,
    composition, end-use.
  \item Do NOT mention codes, category numbers, or any trade/tariff
    references.
  \item Do NOT begin with preambles such as ``Based on the
    classification\ldots''; state the fact directly.
  \item Keep the answer to 1--2 sentences.
\end{enumerate}

\medskip
\textbf{Examples:}

\emph{Q: ``Is the container size over 2~kg?''}\\
\tgreen{\checkmark} ``Yes, the product is packaged in containers
exceeding 2~kg.''\\
\tgray{$\times$ ``The product description does not specify the
container size.''}

\emph{Q: ``Does this contain dairy products?''}\\
\tgreen{\checkmark} ``No, this product does not contain dairy
products or milk solids.''\\
\tgray{$\times$ ``Based on the classification hierarchy\ldots''}

\emph{Q: ``Is this retail candy or confectionery?''}\\
\tgreen{\checkmark} ``Yes, this is confectionery for retail
consumption.''\\
\tgray{$\times$ ``The description does not explicitly state this.''}

\medskip
\textbf{Answer} (factual product attributes only, NO classification
preamble, NO hedging):
\end{tcolorbox}

\noindent\textbf{Implementation notes and oracle design rationale.}
The \texttt{\{hts\_node\_descriptions\_for\_current\_item\}} field
is populated from the HTS knowledge graph node descriptions along
the item's ground-truth classification path.
All numeric HTS codes are replaced by generic category labels
(``category 1'' through ``category 5'') via a regex filter before
injection, and a second pass is applied to the generated answer
to mask any residual HTS code references.
These filters remove explicit identifiers but do not eliminate
semantic information derived from the correct path.

\smallskip
\noindent\textbf{Knowledge-source characterization.}
The oracle draws on HTS node descriptions along the ground-truth
path to produce product-attribute answers.
As confirmed by the knowledge-channel audit (\S\ref{sec:separability},
Appendix~\ref{app:knowledge_audit}), the channel primarily conveys
product-owner knowledge rather than classification paths, so results
measure \emph{help-seeking and gating} behavior in isolation from
answer quality.
The results should be interpreted as measuring \emph{where} the agent
chooses to seek help and how that help-seeking affects navigation
accuracy.
Replacing this controlled oracle with deployment-realistic
information sources (product databases, manufacturer specifications)
is an important direction for future work.

\subsection{Navigation Prompt (Baseline)}
\label{app:nav_prompt}

\begin{tcolorbox}[colback=gray!5, colframe=gray!50,
  title=Navigation Agent System Prompt, breakable]
\footnotesize

You are an expert HTS classification agent navigating the Harmonized
Tariff Schedule tree structure.

\texttt{[\{clarifications\_section\}]} (injected when prior Q\&A
exists; includes per-node duplicate guard and hard cap at 2
clarifications per node)

\medskip
\textbf{GENERAL RULES OF INTERPRETATION (GRI)}

\textbf{GRI~1:} Classification is determined by heading terms and
relative section/chapter notes.
\textbf{GRI~2:} Incomplete articles and mixtures follow the essential
character of the complete article or primary constituent.
\textbf{GRI~3:} When classifiable under two or more headings:
(a)~most specific prevails; (b)~essential character for mixtures/sets;
(c)~last heading in numerical order.
\textbf{GRI~4:} Classify under the heading for the most akin goods.
\textbf{GRI~5:} Special rules for containers and packing materials.
\textbf{GRI~6:} Subheading classification applies GRI~1--5 mutatis
mutandis.
\textbf{Additional U.S.\ Rules:} Classification by principal use at
importation; ``parts'' provisions cover goods solely/principally used
as parts.

\medskip
\textbf{NAVIGATION CONTEXT}

\noindent Parent Node: \texttt{\{parent\_code\}: \{parent\_desc\}}\\
Current Node: \texttt{\{current\_code\}: \{current\_desc\}}\\
Product: \texttt{\{product\_description\}}\\
Path History: \texttt{\{history\}}

\medskip
\textbf{AVAILABLE DIRECT CHILDREN:}\\
\texttt{- \{code\}: \{description\}} (one per line; pruned nodes
listed separately with warning)

\medskip
\textbf{AVAILABLE ACTIONS:}\\
1.~\texttt{traverse\_child(code)}: descend to a child node\\
2.~\texttt{backtrack}: return to parent\\
3.~\texttt{need\_clarify(question)}: ask a product-attribute
question (not HTS/trade references)\\
4.~\texttt{jump(code)}: cross-tree navigation via exclusion edges\\
5.~\texttt{confirm}: declare final classification
(only at 10-digit terminal leaf node)

\medskip
\textbf{Respond with JSON:}
\begin{verbatim}
{"action_type": "traverse_child",
 "target": "code", "question": null,
 "reasoning": "why this child"}
{"action_type": "need_clarify", "target": null,
 "question": "specific product attribute question",
 "reasoning": "which children this resolves"}
{"action_type": "confirm", "target": null,
 "question": null,
 "reasoning": "why confirming at this 10-digit leaf"}
\end{verbatim}
\end{tcolorbox}

\subsection{Action Rating Section (appended when \method{} enabled)}
\label{app:rating_prompt}

\begin{tcolorbox}[colback=gray!5, colframe=gray!50,
  title=Action Rating Section (appended to Navigation Prompt), breakable]
\footnotesize

\textbf{ACTION RATING} (required, include in every response)

From ALL available actions listed above (traverse\_child options,
backtrack, need\_clarify, jump, confirm), identify your TOP $K$
most relevant actions and rate each 0--100:

\medskip
\hspace{1em}100 = definitely the right move at this node\\
\hspace{2.2em}0 = completely wrong / would derail classification\\
\hspace{1em}50 = uncertain / could go either way

\medskip
\textbf{Format each action description as:}\\
\texttt{traverse\_child to \{code\} (\{desc\})}\\
\texttt{backtrack to \{parent\_code\} (\{parent\_desc\})}\\
\texttt{need\_clarify about \{topic\}}\\
\texttt{confirm code \{code\}} $\cdot$
\texttt{jump to \{code\} (\{desc\})}

\medskip
\textbf{Example} (do not copy these scores):
\begin{verbatim}
traverse_child to 8418 (Refrigerators): 92/100
  because: product is clearly a refrigerating appliance
need_clarify about cooling capacity: 45/100
  because: capacity determines the correct 8-digit code
backtrack to 84 (Machinery): 3/100
  because: current node is already specific enough
\end{verbatim}

\textbf{Include ratings in JSON as} \texttt{"action\_ratings"}
(ranked highest $\to$ lowest):
\begin{verbatim}
"action_ratings": [
  {"action_description": "traverse_child to 8418 ...",
   "score": 92,
   "reason": "product is a refrigerating appliance"},
  ...  (exactly K entries)
]
\end{verbatim}
\end{tcolorbox}

\end{document}